\documentclass[twoside]{article}

\usepackage[accepted]{arxiv}

\usepackage{hyperref}
\usepackage[round]{natbib}

\usepackage{natbib}
\usepackage{amsmath}
\usepackage{amssymb}
\usepackage{amsthm}
\usepackage{algorithm}
\usepackage{algpseudocode}
\usepackage{graphicx}

\newcommand{\Pb}{\mathbb{P}}
\newcommand{\Nb}{\mathbb{N}}
\newcommand{\R}{\mathbb{R}}
\newcommand{\E}{\mathbb{E}}

\newcommand{\Hh}{\mathcal{H}}
\newcommand{\F}{\mathcal{F}}
\newcommand{\wealth}{\mathcal{K}}
\newcommand{\payoff}{\mathcal{S}}
\newcommand{\defined}{:=}
\newcommand{\lp}{\left(}
\newcommand{\rp}{\right)}
\newcommand{\lb}{\left[}
\newcommand{\rb}{\right]}
\newcommand{\rba}{\right|}
\newcommand{\lba}{\left|}
\newcommand{\cas}{\stackrel{a.s.}{\to}}

\newtheorem{definition}{Definition}
\newtheorem{theorem}{Theorem}

\newtheorem{proposition}{Proposition}

\begin{document}

%

%

\twocolumn[

\aistatstitle{Sequential Kernelized Stein Discrepancy}

\aistatsauthor{ Diego Martinez-Taboada \And Aaditya Ramdas }

\aistatsaddress{ Department of Statistics \& Data Science \\
Carnegie Mellon University \\
diegomar@andrew.cmu.edu
\And  
Department of Statistics \& Data Science \\
Machine Learning Department \\
Carnegie Mellon University\\
aramdas@stat.cmu.edu} ]

\begin{abstract}
  
We present a sequential version of the kernelized Stein discrepancy goodness-of-fit test, which allows for conducting goodness-of-fit tests for unnormalized densities that are continuously monitored and adaptively stopped. That is, the sample size need not be fixed prior to data collection; the practitioner can choose whether to stop the test or continue to gather evidence at any time while controlling the false discovery rate. In stark contrast to related literature, we do not impose uniform boundedness on the Stein kernel. Instead, we exploit the potential boundedness of the Stein kernel at arbitrary point evaluations to define test martingales, that give way to the subsequent novel sequential tests. We prove the validity of the test, as well as an asymptotic lower bound for the logarithmic growth of the wealth process under the alternative.  We further illustrate the empirical performance of the test with a variety of distributions, including restricted Boltzmann machines.  
\end{abstract}

\section{INTRODUCTION} \label{sec:introduction}

Many statistical procedures heavily rely on the assumptions made about the distribution of the empirical observations, and prove invalid if such conditions are violated. For instance, a high-energy physicist may develop a generative model seeking to produce synthetic observations of particle collisions, motivated by the high cost of experimenting on a real particle collider \citep{agostinelli2003geant4, chekalina2019generative}. If the model is not accurate, then any conclusion that derives from such synthetic data will lack any scientific interest \citep{huang2023learning, masserano2023simulator}. The question of whether the data follows a particular distribution may be posed in terms of goodness-of-fit testing \citep{lehmann1986testing, gonzalez2011general, d2017goodness}. Formally, the simple goodness-of-fit testing problem considers the null hypothesis $H_0: Q = P$ against the alternative $H_1: Q \neq P$, for a given $P$ and access to data $X_1, X_2, \ldots \sim Q$. It may also be the case that we have a set of given distributions as candidates and we wish to test if \textit{any} of them accurately models the empirical data \citep{durbin1975kolmogorov, key2021composite}. In such a case, the goodness-of-fit problem is posed in terms of a composite null hypothesis $H_0: Q \in \mathcal{P}$, against the alternative hypothesis $H_1: Q \not\in \mathcal{P}$.

In this work, we focus on a particularly challenging scenario that arises from not having full access to $P$ (or $\mathcal{P}$). We handle classes of distributions whose densities are only known up to normalizing constants. This is a common scenario when working with general energy-based models \citep{lecun2006tutorial, du2019implicit}, such as Ising models \citep{geman1984stochastic, clifford1990markov} and restricted Boltzmann machines \citep{ackley1985learning, hinton2010practical}, and in Bayesian statistics, where normalizing factors are generally neither known in closed form nor computable \citep{bolstad2016introduction}. Most of the existing works in the literature regarding goodness-of-fit for unnormalized densities have focused on the \textit{batch} or \textit{fixed sample size} setting, where the number of samples $n$ is decided before conducting the analysis, which is later on carried on using observations $X_1, \ldots, X_n$ \citep{liu2016kernelized, chwialkowski2016kernel, gorham2017measuring}. Nonetheless, this approach comes with several major drawbacks. Because the number of observations in batch tests is determined prior to data collection, there is a risk that too many observations may be allocated to simpler problem instances, thus wasting resources, or that insufficient observations are assigned to more complex instances, which can yield insufficient evidence against the null hypothesis. Furthermore, when test outcomes appear encouraging but not definitive (for instance, if a p-value is marginally higher than a specified significance threshold), one may be tempted to augment the dataset and undertake the investigation again. Traditional batch testing methods, however, do not allow for this approach.

In contrast, we seek to develop tests that are anytime valid \citep{ramdas2023game, grunwald2020safe}. That is, we can repeatedly decide whether to collect more data based on the current state of the procedure without compromising later assessments, halting the process at any time and for any reason. Formally, a sequential test which is stopped at any given arbitrary moment can be represented by a random stopping time $\tau$ taking values in $\{ 1, 2, \ldots \} \cup \{ \infty\}$. The stopping time $\tau$ denotes the random time at which the null hypothesis is rejected. A sequential test is called level-$\alpha$ if it satisfies the condition $\Pb (\tau < \infty) \leq \alpha$ under the null for any stopping time $\tau$. 

In this work, we develop level-$\alpha$ sequential goodness-of-fit tests that accommodate for unnormalized densities, building on the concept of kernelized Stein discrepancies. The type I error is controlled even if the tests are continuously monitored and adaptively stopped, hence automatically adapting the sample size to the unknown alternative. In particular, we believe these anytime-valid tests will be of remarkable interest in the following scenarios: 
\begin{itemize}
\setlength\itemsep{0em}
    \item \textbf{Measuring the quality of a proposed distribution:} We may have a specific candidate distribution (or set of candidate distributions) to model the data whose normalizing constant is unknown (e.g., some complex network that models the interaction between genes). Does this distribution accurately model the true data generation process? 
    \item \textbf{Measuring sample quality:} We may have obtained an unnormalized density, such as some posterior distribution in Bayesian statistics, that we wish to draw from using Markov Chain Monte Carlo (MCMC) methods. Is the chosen MCMC algorithm yielding samples accurately from such a density?\footnote{There exist complexities when sequentially analyzing the fitness of MCMC samples, mainly related to autocorrelation and burn-in stages. In order to minimize the effect of autocorrelation, one could skip samples and only input one every X samples for the sequential test, where X is large enough to considerably lower the correlation between samples. In order to avoid early rejection due to burn-in stage, sequential tests can be initialized after such a burn-in period.}
\end{itemize}
The anytime-valid guarantees allow for minimizing the sample size required to reject the null if it does not hold, such as the number of expensive gene experiments that we have to run in a lab, or the number of samples that we have to draw from a costly MCMC algorithm. 

The work is organized as follows. Section \ref{sec:problem_setup} presents the related work. Section \ref{sec:preliminaries} introduces preliminary work, with special emphasis on the kernelized Stein discrepancy and testing by betting. These constitute the theoretical foundations of the novel goodness-of-fit tests, presented in Section \ref{sec:main_results}. Section \ref{sec:derivation} subsequently provides a general derivation for the lower bounds required in the test, alongside different examples. The empirical validity of the tests is presented in Section \ref{sec:experiments}, followed by concluding remarks in Section \ref{sec:conclusion}.

\section{RELATED WORK} \label{sec:problem_setup}

Our contribution falls within the scope of `sequential, anytime-valid inference', a field which encompasses confidence sequences \citep{waudby2024estimating} and e-processes \citep{ramdas2022testing, grunwald2020safe}. In particular, we exploit the techniques of `testing by betting', which was recently popularized by \cite{shafer2021testing}, and is based on applying Ville’s maximal inequality \citep{ville1939etude} to a test (super)martingale \citep{shafer2011test}. Any other approach for constructing confidence sequences can be, in principle, outperformed by applying maximal inequalities to (super)martingale constructions \citep{ramdas2020admissible}. The interest on this  line of research has recently exploded; we refer to reader to \citet{ramdas2023game} and the references therein for a detailed introduction to the field. 

Concurrently, kernel methods have received widespread attention due to their notable empirical performance and theoretical guarantees. Reproducing kernel Hilbert spaces provide the theoretical foundation for these methods \citep{scholkopf2002learning, berlinet2011reproducing}. We highlight three predominant applications of kernel methods. First, the maximum mean discrepancy (MMD) has given way to kernel-based two sample tests \citep{gretton2006kernel, gretton2012kernel}. Second, independence tests have been developed based on the Hilbert Schmidt independence criterion (HSIC) \citep{gretton2005measuring, gretton2007kernel}. Third, kernelized Stein discrepancies (KSD) have led to novel goodness-of-fit tests \citep{liu2016kernelized, chwialkowski2016kernel, gorham2017measuring}. While the analysis of this contribution is tailored to so-called Langevin KSD, the ideas presented herein could potentially be extended to abstract domains including categorical data \citep{yang2018goodness}, censored data \citep{fernandez2020kernelized}, directional data \citep{xu2020stein}, functional data \citep{wynne2025fourier}, sequential data \citep{baum2023kernel}, and point processes \citep{yang2019stein}. 

At the intersection of testing by betting and kernels, \citet{shekhar2023nonparametric} developed a sequential MMD for conducting two sample tests that can be arbitrarily stopped. Similarly, \citet{podkopaev2023sequential} extended the HSIC to the sequential setting. Recently, \citet{zhou2024sequential} proposed a sequential KSD, in a similar spirit to our contribution.\footnote{We defer a comprehensive comparison of \citet{zhou2024sequential} and our contribution to Appendix~\ref{appendix:comparison}.} These three works rely on the uniform boundedness of the chosen kernel, which is key for the construction of nonnegative martingales. While some kernels commonly used for the MMD and HSIC (such as the RBF or Laplace kernels) are uniformly bounded, this is very rarely the case for the Stein kernel exploited by the KSD. The methodology proposed in this contribution does not assume uniform boundedness of the Stein kernel, which poses a number of theoretical challenges that will be addressed in the subsequent sections.

\section{BACKGROUND} \label{sec:preliminaries}

\subsection{The Kernelized Stein Discrepancy}

Throughout this contribution we will draw upon the kernelized Stein discrepancy (KSD), which is a distributional divergence that builds on the concept of reproducing kernel Hilbert space (RKHS) and the general Stein's method \citep{stein1972bound, gorham2015measuring}. The KSD may be understood as a kernelized version of score-matching divergence \citep{hyvarinen2005estimation}.

\textbf{Reproducing kernel Hilbert spaces (RKHS):} Consider $\mathcal{X} \neq \emptyset$ and a Hilbert space $(\mathcal{H}, \langle \cdot, \cdot \rangle_\mathcal{H})$ of functions $f: \mathcal{X} \to \mathbb{R}$. If there exists a function $k: \mathcal{X} \times \mathcal{X} \to \mathbb{R}$ satisfying (i) $k(\cdot, x) \in \mathcal{H}$ for all $x \in \mathcal{X}$, (ii) $\langle f, k(\cdot, x) \rangle_\mathcal{H} = f(x)$ for all $x \in \mathcal{X}$ and $f \in \mathcal{H}$, then $\mathcal{H}$ is called an RKHS and $k$ a reproducing kernel. We denote by $\mathcal{H}^d$ the product RKHS containing elements $h := (h_1, \ldots, h_d)$ with $h_i \in \mathcal{H}$ and $\langle h, \tilde h \rangle = \sum_{i = 1}^d \langle h_i, \tilde h_i\rangle_{\mathcal{H}}$.

\textbf{Kernelized Stein discrepancies (KSD):} Assume now that $P$ has density $p$ on $\mathcal{X} \subset \R^d$, and let $\mathcal{H}$ be an RKHS with reproducing kernel $k$ such that $\nabla k(\cdot, x)$ exists for all $x \in \mathcal{X}$. Based on the Stein operator 
\begin{align*}
    (T_{p} h)(x) = \sum_{i = 1}^d\left(\frac{\partial \log {p} (x)}{\partial x_i}h_i(x) + \frac{\partial h_i(x)}{\partial x_i} \right), \; h \in \Hh^d,
\end{align*} 
the KSD is defined as
\begin{align*}
    \text{KSD}_\Hh(Q, P) = \sup_{f\in\F_{\text{KSD}}} \E_{X \sim Q}[(T_{p}f)(X)],
\end{align*}
where we consider the set $\F_{\text{KSD}} = \{ h \in \Hh^d: \|h\|_{\Hh^d} \leq 1 \}$. Defining $s_p(x) := \nabla_x \log p(x)$ and $\xi_p(\cdot, x) := \left[ s_p(x) k(\cdot, x) + \nabla k(\cdot, x)\right] \in \mathcal{H}^d$, it follows that
\begin{align} 
    h_p(x, \tilde x) :&= \langle \xi_p (\cdot, x), \xi_p (\cdot, \tilde x) \rangle_{\mathcal{H}^d}\nonumber\\ & =\langle s_p(x), s_p(\tilde x) \rangle_{\R^d} k(x, \tilde x)\nonumber\\
    &\quad+ \langle s_p(\tilde x), \nabla_{x} k(x, \tilde x) \rangle_{\R^d} \nonumber\\
    &\quad + \langle s_p(x), \nabla_{\tilde x} k(x, \tilde x) \rangle_{\R^d} + \nabla_{x} \cdot \nabla_{\tilde x} k(x, \tilde x) \label{eq:closedform} 
\end{align}
is a reproducing kernel based on the Moore-Aronszajn theorem. Again, if $\E_{X\sim Q} \sqrt{h_p(X, X)} < \infty$, then $\text{KSD}_\Hh(Q, P) = \| \E_{X \sim Q} \left[\xi_p(\cdot, X)\right]\|_{\mathcal{H}^d}$. If $k$ is universal and under certain regularity conditions, then $\text{KSD}_\Hh(Q, P) = 0$ if, and only if, $Q = P$ \citep{chwialkowski2016kernel}. In the batch setting and simple null hyposthesis, a test statistic is generally taken as a V-statistic or U-statistic, and parametric or wild bootstrap is used to calibrate the test. We refer the reader to \citet{key2021composite} for the more challenging composite null hypothesis case.\footnote{In the context of approximating an integral, the KSD also allows for evaluating the realized empirical measure (rather than the marginal distribution). We refer the reader to \citet{kanagawa2022controlling} for a recent account of this research direction.}

\subsection{Testing by Betting}

We now present the theoretical foundation of the sequential testing strategy that will be presented in Section \ref{sec:main_results}, which is commonly referred to as \textit{testing by betting}. The key idea behind this general, powerful concept relies on defining a test (super)martingale, and couple it with a betting interpretation \citep{shafer2019game, shafer2021testing}. Test (super)martingales allow for exploiting Ville's inequality \citep{ville1939etude} and hence they yield level-$\alpha$ sequential tests. In turn, this enables practitioners to peek at the data at anytime and decide whether to keep gathering data or stop accordingly while preserving theoretical guarantees. 

Intuitively, let $H_0$ be the null hypothesis to be tested. A fictional bettor starts with an initial wealth of $\wealth_0=1$. The fictional bettor then bets sequentially on the outcomes $(X_t)_{t \geq 1}$. To do so, at each round $t$, the fictional bettor chooses a payoff function $\payoff_t: \mathcal{X} \to [0, \infty)$ so that $\mathbb{E}_{H_0}[\payoff_t(X_t)|\F_{t-1}] \leq 1$,  where $\F_{t-1} = \sigma(X_1, \ldots, X_{t-1})$ (this ensures a \textit{fair bet} if the null is true). 
After the outcome $X_t$ is revealed, the bettor's wealth grows or shrinks by a factor of $\payoff_t(Z_t)$. The bettor's wealth is $\wealth_t =\wealth_0 \prod_{i=1}^t \payoff_i(X_i)$ after $t$ rounds of betting.

Under the null hypothesis, these \textit{fair bets} ensure that the sequence $(\wealth_t)_{ t\geq 0}$ forms a `test supermartingale' (i.e., a nonnegative supermartingale that starts at 1), and in particular a `test martingale' if $\mathbb{E}_{H_0}[\payoff_t(X_t)|\F_{t-1}]$ is $1$ almost surely. Based on Ville's inequality, rejecting the null if $\wealth_t \geq 1/\alpha$, where $\alpha \in (0,1)$ is the desired confidence level, leads to sequential tests that control the type-I error at level $\alpha$. Under the alternative $H_1$, the payoff functions $\{\payoff_t: t \geq 1\}$ should seek a fast growth rate of the wealth, ideally exponentially.

\section{SEQUENTIAL GOODNESS-OF-FIT BY BETTING} \label{sec:main_results}

We now consider a stream of data $X_1, X_2, \ldots \sim Q$. In the context of testing by betting, it is natural to consider test martingales of the form
\begin{align}
    &\wealth_t = \wealth_{t-1} \times \lp 1 + \lambda_t g_t(X_t)  \rp, \quad \wealth_0 = 1, \label{eq:wealth-process-def} \\
    &\tau \defined \min \{t \geq 1: \wealth_t \geq 1/\alpha\} \label{eq:two-sample-test-def}, 
\end{align}
where $\lambda_t$ and $g_t$ are predictable and
\begin{itemize}
    \item $\mathbb{E}_{H_0}\left[ g_t(X_t) \big|\F_{t-1}\right] \leq 0$ (to ensure $\wealth_t$ is a supermartingale under the null hypothesis),
    \item $g_t \geq -1$ and $\lambda_t \in [0, 1]$ (to ensure that $\wealth_t$ is non-negative),
\end{itemize} 
so that $\wealth_t$ is a test supermartingale. Intuitively, $g_t(X_t)$ defines the payoff function, and $\lambda_t$ the proportion of the wealth that the bettor is willing to risk at round $t$. Usually, the payoff function $g_t$ is assumed to belong to a set $\mathcal{G}$ that is \textbf{uniformly bounded by one in absolute value}  and such that $\mathbb{E}_{H_0}\left[ g(X) \big|\F_{t-1}\right] \leq 0$ for all $g \in \mathcal{G}$, and is chosen predictably seeking to maximize $\wealth_t$ under the alternative.

For the Stein kernel $h_p$, we note that, if  $\E[h_p(X, X)] < \infty$ \citep[Lemma 5.1]{chwialkowski2016kernel}, 
\begin{align} \label{eq:ksd_zero_mean}
    &\mathbb{E}_{H_0}\left[ \frac{1}{t-1} \sum_{i = 1}^{t-1} h_p(X_i, X_t)\bigg|\F_{t-1}\right] 
    \\=&\mathbb{E}_{H_0}\left[ \frac{1}{t-1} \sum_{i = 1}^{t-1} \langle \xi_p (\cdot, X_i), \xi_p(\cdot, X_t) \rangle_{\mathcal{H}^d} \bigg|\F_{t-1}\right] \nonumber\\
    = &\frac{1}{t-1} \sum_{i = 1}^{t-1} \langle \xi_p (\cdot, X_i), \mathbb{E}_{H_0}\left[  \xi_p (\cdot, X_t)  \big|\F_{t-1}\right] \rangle_{\mathcal{H}^d} \nonumber\\
    \stackrel{}{=} &\frac{1}{t-1} \sum_{i = 1}^{t-1} \langle \xi_p (\cdot, X_i), 0 \rangle_{\mathcal{H}^d} \nonumber\\
    \stackrel{}{=} &0. \nonumber
\end{align}
Hence, if $h_p$ was uniformly bounded by one, we could define $g_t(x) = \frac{1}{t-1}\sum_{i = 1}^{t-1} h_p(X_i, x)$. Similar payoff functions have been proposed in \citet{shekhar2023nonparametric} for two sample testing, and in \citet{podkopaev2023sequential} for independence testing. In those settings, it is natural to consider kernels that are indeed uniformly bounded by one, and this bound is not too loose\footnote{\citet{shekhar2023nonparametric} proposed an extension to unbounded kernels. They exploit the symmetry (around zero) of their payoff function under the null, which follows from the nature of two sample testing. Such a symmetry does not hold for us.}. This is the case for the ubiquitous Gaussian (or RBF) and Laplace kernels. 

In stark contrast, the Stein kernel $h_p$ need not be uniformly bounded even if built from a uniformly bounded $k$ (or it may be uniformly bounded by an extremely large constant, which would imply a remarkable loss in power if simply normalizing by that constant). Consider the inverse-multi quadratic (IMQ) kernel
\begin{align*}
    k_{\text{IMQ}}(x, y) = (1 + \|x - y \|^2)^{-1/2}
\end{align*}
The IMQ kernel has been extensively employed as the base kernel for constructing the Stein kernel. The success of the IMQ kernel over other common characteristic kernels can be attributed to its slow decay rate \citep{gorham2017measuring}. For this reason, we build the Stein kernel $h_p$ from $k_{\text{IMQ}}$ throughout. Now take $P = \mathcal{N}(0, 1)$ to be a standard normal distribution.
In such a case, $s_p(x) = -x$, and $h_p(x, x) = x^2 + 1$. This implies that $\sup_{x \in \R} h_p(x, x) = \infty$, i.e., the kernel $h_p$ is not bounded.

Nonetheless, it is very possible that $x \mapsto h_p(x, \tilde x)$ is bounded for every fixed $\tilde x$. We can rewrite 
\begin{align*}
    h_p(x, \tilde x) &= \langle \xi_p (\cdot, x), \xi_p (\cdot, \tilde x) \rangle_{\mathcal{H}^d} 
    \\&= \|\xi_p (\cdot, x)\|_{\Hh^d} \|\xi_p (\cdot, \tilde x)\|_{\Hh^d} \cos \beta(x, \tilde x),
\end{align*}
where $\cos \beta(x, \tilde x)$ is the angle between $\xi_p (\cdot, x)$ and $\xi_p (\cdot, \tilde x)$ in the Hilbert space ${\Hh^d}$ . Interestingly, while the embeddings $\|\xi_p (\cdot, x)\|_{\Hh^d}$ may not be uniformly bounded (this is equivalent to $h_p$ not being uniformly bounded), $\cos \beta(x, \tilde x)$ may decay to zero faster than $\|\xi_p (\cdot, x)\|_\Hh$ approaches infinity (for any given $\tilde x$).  From now on, we assume that we have access to a function $M_p: \mathcal{X} \mapsto \R_{\geq 0}$ such that  
\begin{align*}
    M_{p}(\tilde x) \geq -\inf_{x \in \mathcal{X}} h_p(x, \tilde x).
\end{align*}

The upper bound $M_p$ plays a key role in the proposed test, as it allows for defining martingales that are nonnegative. We defer to Section \ref{sec:derivation} a general method to derive such an upper bound, as well as specific examples of such derivation.

\subsection{Simple Null Hypothesis} \label{subsec:unnormalized}

 We are now ready to present the novel sequential goodness-of-fit tests. Let $P$ be a distribution which is known up to its normalizing constant, and let us consider the null hypothesis $H_0: Q = P$ against the alternative $H_1: Q \neq P$. Note that the scores $s_p(x)$ of $P$ do not depend on such normalizing factors and thus the reproducing kernel $h_p$ is computable. 
 
 Again, we emphasize that defining a wealth process based on the payoff function $\frac{1}{t-1} \sum_{i = 1}^{t-1} h_p(X_i, x)$ does not, in principle, yield a nonnegative martingale under the null. Nonetheless, the normalized payoff function
\begin{align} \label{eq:gt_definition}
    g_t(x) = \frac{1}{\frac{1}{t-1}\sum_{i=1}^{t-1} M_{p}(X_i)}\left(\frac{1}{t-1} \sum_{i = 1}^{t-1} h_p(X_i, x)\right).
\end{align}
is lower bounded by -1. Hence, as long as the betting strategy $\lambda_t$ is nonnegative and bounded above by 1, we are able to define a wealth process that forms a test martingale. The following theorem formally establishes such a result; the proof is deferred to Appendix \ref{appendix:proofs}.

\begin{theorem}[Validity under null.] \label{theorem:ksd_simple}
Assume that $\E_{H_0}[h_p(X, X')]=0$, and let $\lambda_t \in [0, 1]$ be predictable. The wealth process 
\begin{align} \label{eq:wealth_definition}
    &\wealth_t = \wealth_{t-1} \times \lp 1 + \lambda_t g_t(X_t)  \rp, \quad \wealth_0 = 1,
\end{align}
where $g_t$ is defined as in~\eqref{eq:gt_definition},
is a test martingale. The stopping time
\begin{align*}
     &\tau \defined \min \{t \geq 1: \wealth_t \geq 1/\alpha\}
\end{align*}
defines a level-$\alpha$ sequential test.
\end{theorem}

Algorithm~\ref{alg:ksd} summarizes the procedure introduced in this section. Given that the complexity of computing $g_t$ is $O(t)$ for each $t$, the complexity of Algorithm~\ref{alg:ksd} for $T$ rounds is $O(T^2)$. Note that $\E[h_p(X, X')] = 0$ under the null if $\E[h_p(X, X)] < \infty$, and thus in that case Theorem \ref{theorem:ksd_simple} establishes the validity of Algorithm \ref{alg:ksd} for any betting strategy that is predictable. However, the power of the test will heavily depend on the chosen betting strategy. For instance, if we take $\lambda_t = 0$ for all $t$ (this is, we never bet any money), then the wealth will remain constant as $\wealth_t = 1$, and so the test is powerless.  There exists a variety of betting strategies that have been studied in the literature. In this contribution, we focus on aGRAPA (`approximate GRAPA') and LBOW (`Lower-Bound On the Wealth'). We refer the reader to \cite{waudby2024estimating} to a detailed presentation of different betting strategies.

\begin{definition} [aGRAPA strategy]
Let $\lp g_i(X_i)\rp_{t \geq 1} \in [-1, \infty)^{\Nb}$ denote a sequence of outcomes. Initialize $\lambda_1^{\text{aGRAPA}} = 0$. For each round $t = 1, 2, \ldots$, observe payoff $g_i(X_i)$ and update 
\begin{align} \label{eq:agrapa}
    \lambda_{t+1}^{\text{aGRAPA}} = 1 \wedge \lp0 \vee \frac{\frac{1}{t-1}\sum_{i = 1}^{t-1} g_i(X_i)}{ \frac{1}{t-1}\sum_{i = 1}^{t-1} g_i^2(X_i)}\rp.
\end{align}
\end{definition}

\begin{definition} [LBOW strategy]
Let $\lp g_i(X_i)\rp_{t \geq 1} \in [-1, \infty)^{\Nb}$ denote a sequence of outcomes. Initialize $\lambda_1^{\text{LBOW}} = 0$. For each round $t = 1, 2, \ldots$, observe payoff $g_i(X_i)$ and update 
\begin{align} \label{eq:lbow}
    \lambda_{t+1}^{\text{LBOW}} = 0\vee \frac{\frac{1}{t-1}\sum_{i = 1}^{t-1} g_i(X_i)}{\frac{1}{t-1}\sum_{i = 1}^{t-1} g_i(X_i) + \frac{1}{t-1}\sum_{i = 1}^{t-1} g_i^2(X_i)}.
\end{align}
\end{definition}

We have introduced versions of the betting strategies where we force $\lambda_t$ to be nonnegative. This need not always be the case. The idea of not allowing for negative bets was exploited by \citet{podkopaev2023sequential} as well, motivated by the fact that positive payoffs are expected under the alternative. The motivation is double in this work, as we also expect positive payoffs under the alternative, but the nonnegativity of $\lambda_t$ allows us to only having to lower bound the payoff function $g_t$, instead of lower and upper bounding it.  While the aGRAPA strategy shows better empirical performance, LBOW allows for providing theoretical guarantees of the asymptotic wealth under the alternative.

\begin{theorem}[E-power under alternative]
\label{theorem:ksd_simple_power_lbow}

Let $(X_i)_{i\geq1}$ be independent and identically distributed copies of $X$. Assume $\E[h_p(X, X)] < \infty$ and $\E[M_{p}(X)] < \infty$. Denote $g^*(x) := \E[ h_p(X, x)] / \E[M_p(X)]$. Under $H_1$, if $\E[g^*(X)] > 0$, the LBOW betting strategy yields 
\begin{align*}
    \liminf_{t \to \infty}\frac{\log\wealth_t}{t}  \geq \frac{\left(\E[g^*(X)]\right)^2 / 2}{\E[g^*(X)] + \E[\left(g^*(X)\right)^2]} := r^*.
\end{align*}
It follows that $\wealth_t \stackrel{a.s.}{\to} \infty$ and $\Pb_{H_1} (\tau < \infty) = 1$.
\end{theorem}

Informally, the above theorem states that under the alternative, $\wealth_t \geq \exp(r^* t (1-o(1)))$, meaning that up to asymptotically negligible terms, the wealth grows exponentially fast (in the number of data points $t$) at the rate  $r^*$. Note that under the null, $\E[h_p(X, X')] = 0$, implying $\E[g^*(X)] = 0$, and thus $r^* = 0$, which accords with our claim that under the null, the wealth is a nonnegative martingale (whose expectation stays constant with $t$) and thus does not grow with $t$. This then implies that the stopping time of the test, which is the time at which the wealth exceeds $1/\alpha$, is (up to leading order) given by the expression $\log(1/\alpha)/r^*$.

The proof of Theorem \ref{theorem:ksd_simple_power_lbow} is deferred to Appendix \ref{appendix:proofs}. We point out that the unboundedness of the Stein kernel does not allow for easily extending the arguments presented in \citet{podkopaev2023sequential}, which rely on a different betting strategy whose guarantees stem from uniformly bounded payoff functions. Furthermore, we highlight the mildness of the assumptions in Theorem \ref{theorem:ksd_simple_power_lbow}. Assumption $\E[M_{p}(X)] < \infty$ only requires the first moment of the upper bounds to exist. This condition usually reduces to the existence of a specific moment of the original distribution, which is often easily verifiable.   Assumption $\E[h_p(X, X)] < \infty$ is an ubiquitous assumption in the KSD theory; it is equivalent to the existence of the second moment of $\|\xi_p(\cdot, X)\|_{\mathcal{H}^d}$, which is precisely the theoretical object that the KSD builds on. Finally, note that $P \neq Q$ does not necessarily imply $\E[g^*(X)] > 0$. However, there exist sufficient conditions for this implication to hold,  such as $C_0$-universality of $k$ and a moment condition on $\nabla \log (p(X)/q(X)$ \citep[Theorem 2.2.]{chwialkowski2016kernel} or a root exponential growth of $\| s_p \|$ for a rich family of base kernels \citep[Application 2]{barp2024targeted}. These conditions have already been extensively studied in the batch setting \citep{chwialkowski2016kernel, liu2016kernelized, gorham2017measuring, barp2024targeted}, and equally apply to our setting.


\begin{algorithm}[!htb]
\caption{Sequential KSD}\label{alg:ksd}
\begin{algorithmic}
\State \textbf{Input:} Significance level $\alpha$; data stream $X_1, X_2, \ldots \sim Q$, score function $s_p$, base kernel $k$.
\State Define $h_p$ from $s_p$ and $k$ as in \eqref{eq:closedform}.
\For {$t=1,2,\dots$}
\State Observe $X_t$;
\State Compute $g_t(X_t)$ using \eqref{eq:gt_definition}  and $\wealth_t$ using \eqref{eq:wealth_definition};

\If{$\wealth_t\geq 1/\alpha$}
\State Reject $H_0$ and stop;
\Else 
\State Compute $\lambda_{t+1} \in [0, 1]$ following \eqref{eq:agrapa} or \eqref{eq:lbow};
\EndIf

\EndFor
\end{algorithmic}
\end{algorithm}

\subsection{Composite Null Hypothesis} \label{section:composite_null_hypothesis}

Let now $\mathcal{P} = \{ P_\theta: \theta \in \Theta \}$ be a set of distributions with unknown normalizing constants parametrized by $\theta$. By a slight abuse of notation, we denote $s_\theta = s_{p_\theta}$, $h_\theta = h_{p_\theta}$, $M_\theta = M_{p_\theta}$, and so on. Consider the null hypothesis $H_0: Q \in \mathcal{P}$ against $H_1: Q \not\in\mathcal{P}$. We define
\begin{align*}
    g_t^\theta(x) = \frac{1}{\frac{1}{t-1}\sum_{i=1}^{t-1} M_{\theta}(X_i)}\left(\frac{1}{t-1} \sum_{i = 1}^{t-1} h_\theta(X_i, x)\right)
\end{align*}
for $\theta \in \Theta$. We note that, under the null hypothesis, there exists $\theta_0 \in \Theta$ such that $Q = P_{\theta_0}$. Inspired by universal inference, we propose to consider the wealth process
\begin{align*}
    \wealth_t^C = \min_{\theta \in \Theta} \wealth^\theta_t.
\end{align*}

\begin{theorem} \label{theorem:ksd_composite}
    $\wealth_t^C$ is dominated by a test supermartingale, and $\tau$ is a level-$\alpha$ sequential test.
\end{theorem}

The minimizer may be computed differently depending on the nature of the problem. If $\Theta$ is finite, then it is obtained as the minimum of a discrete set. For arbitrary $\Theta$, it may be computed using numerical optimisation algorithms.

\section{DERIVATION OF SENSIBLE BOUNDS} \label{sec:derivation}

\subsection{A General Approach}

We highlight that the derivation of the bounds $M_p(x)$ is key on this approach. While these bounds heavily depend on the distribution $P$ through its score function $s_p$, we explore general ways of deriving bounds $M_p(x)$. We focus on deriving bounds for the IMQ. Nonetheless, we highlight that similar derivations would follow for other kernels, such as the RBF or Laplace kernels. We start by noting that, for the IMQ kernel, 
\begin{align*}
    \nabla_xk(x, \tilde x) &= -(1 + \| x - \tilde x\|^2)^{-\frac{3}{2}}(x - \tilde x), 
    \\ \nabla_{x} \cdot \nabla_{\tilde x} k(x, \tilde x) &= -3(1 + \|x - \tilde x \|^2)^{-\frac{5}{2}}\|x-\tilde x\|^2 
    \\&\quad+ d(1 + \|x - \tilde x \|^2)^{-\frac{3}{2}}. 
\end{align*}
Hence, for a fixed $\tilde x$,
\begin{itemize}
    \item $k(x, \tilde x)$ is  $O(\|x - \tilde x\|^{-1})$,
    \item $\| \nabla_x k(x, \tilde x)\|$ is  $O(\| x - \tilde x\|^{-2})$ ,
    \item $ \nabla_{x} \cdot \nabla_{\tilde x} k(x, \tilde x) \geq \min(-3 + d, 0)$. 
\end{itemize}

In order to explicitly obtain a bound, we work with each of the terms (i) $|\langle s_p(x), s_p(\tilde x) \rangle_{\R^1} k(x, \tilde x)| \leq  \|s_p(x)\|\| s_p(\tilde x) \| | k(x, \tilde x)|$, (ii) $\langle s_p(\tilde x), \nabla_{x} k(x, \tilde x) \rangle_{\R^d} + \langle s_p(x), \nabla_{\tilde x} k(x, \tilde x) \rangle_{\R^d}$, (iii) $ \nabla_{x} \cdot \nabla_{\tilde x} k(x, \tilde x)$ separately. For term (i), we upper bound $\|s_p(x)\|$ by $\gamma(\|x - \tilde x\|) + \gamma'(\tilde x)$, where $\gamma$ and $\gamma'$ are appropriate functions. For term (ii), we note that 
\begin{align*}
    &\langle s_p(\tilde x), \nabla_{x} k(x, \tilde x) \rangle_{\R^d} + \langle s_p(x), \nabla_{\tilde x} k(x, \tilde x) \rangle_{\R^d} 
    \\= &\langle s_p(\tilde x) - s_p(x), -\nabla_{x} k(x, \tilde x) \rangle_{\R^d}
    \\= &\langle s_p(\tilde x) - s_p(x), (1 + \| x - \tilde x \|^2)^{-\frac{3}{2}}(x - \tilde x)\rangle_{\R^d}
    \\= &(1 + \| x - \tilde x |^2)^{-\frac{3}{2}}\langle s_p(\tilde x) - s_p(x), x - \tilde x\rangle_{\R^d},
\end{align*}
so it suffices to upper bound work with $s_p(\tilde x) - s_p(x)$ and upper bound $\|s_p(\tilde x) - s_p(x)\|$ by $\Gamma(\|x - \tilde x\|) + \Gamma'(\tilde x)$, where $\Gamma$ and $\Gamma'$ are again appropriate functions. Note that (iii) has already been lower bounded. 

The choices of $\gamma, \gamma', \Gamma, \Gamma'$ will become clear in the next subsection, where we provide specific examples.  We highlight that there should exist sensible choices as soon as $\|s_p(x)\| k(x, \tilde x)$ and $\|s_p(x) - s_p(\tilde x) \| \| \nabla k(x, \tilde x) \|$ are both $O(1)$ for fixed $\tilde x$, i.e. $\|s_p(x)\| = O(\|x\|)$ in the case of the IMQ kernel. In addition to the cases considered in the next subsection, which include Gaussian distributions and restricted Boltzmann machines, a variety of models may fulfill this condition. In particular, exponential families with canonical parameters $\eta_i$ and canonical observations $T_i$ of the form $p_\theta(x) \propto \exp (\sum_{i \leq m} \eta_i(\theta) T_i(x) )$ fall under this umbrella if $\| \nabla T_i(x)\| = O(\|x\|)$ for all $i \leq m$. For instance, exponential graphical models with linear interactions between nodes, which have found applications in protein signaling networks (see, e.g., \citet[Section 4.2.2.]{yang2015graphical}) and spatial models (see, e.g.,  \citet[Section 4.2.2.]{besag1986statistical} and \citet[Section 4.1.]{besag1974spatial}) among others, attain such an assumption. Kernel exponential family models \citep{canu2006kernel} are another prominent example, which fulfill the condition as long as the gradients of the basis functions of a finite-rank approximation (as considered in \citet{matsubara2022robust}) are $O(\| x \|)$ (which is precisely the case in \citet{matsubara2022robust}). 

\subsection{Specific Examples of Bound Derivations} \label{sec:specific_derivations}

We present specific instances of bound derivations in order to elucidate the use of the general approach. The bounds obtained here will be later used in the experiments displayed in Section \ref{sec:experiments}.

\textbf{Gaussian distribution: } Let us consider $P_\theta = \mathcal{N}(\theta, 1)$, i.e., a normal distribution with mean $\theta$ and unit variance, and $k$ to be the inverse multi-quadratic kernel. Given that $s_\theta(x) = - (x - \theta)$, we first derive 
\begin{align*}
    |\langle s_\theta(x), s_\theta(\tilde x) \rangle k(x, \tilde x)| &\leq | x - \theta | | \tilde x - \theta | k (x, \tilde x) 
    \\ &\leq \lp | x - \tilde x| + | \tilde x -  \theta |\rp | \tilde x - \theta | 
    \\&\quad\times k (x, \tilde x)
    \\&\stackrel{}{\leq} |\tilde x - \theta| \lp 1 + | \tilde x - \theta |\rp,
\end{align*}
where the last inequality can be easily derived by separately considering the two cases $|x - \tilde x| \leq 1$ and $|x - \tilde x| > 1$. Secondly, 
\begin{align*}
    &(1 + | x - \tilde x |^2)^{-\frac{3}{2}}\langle s_p(\tilde x) - s_p(x), x - \tilde x\rangle 
    \\= &-(1 + | x - \tilde x |^2)^{-\frac{3}{2}} | x - \tilde x | \in [-1, 0].
\end{align*}

For a fixed $\tilde x$, it thus follows that i) $|\langle s_\theta(x), s_\theta(\tilde x) \rangle k(x, \tilde x)| \leq |\tilde x - \theta | \{ 1 + | \tilde x - \theta |\}$, ii) $\langle s_\theta(\tilde x), \nabla_{x} k(x, \tilde x) \rangle + \langle s_\theta(x), \nabla_{\tilde x} k(x, \tilde x) \rangle \in [-1, 0]$, iii) $ \nabla_{x} \cdot \nabla_{\tilde x} k(x, \tilde x) \geq -2$. 

Consequently, it suffices to define $M_\theta(\tilde x) = |\tilde x - \theta | \{ 1 + | \tilde x - \theta |\} + 3$. Note that $\E[M_{\theta}(X)] < \infty$ given that the first two moments of a Gaussian distribution are finite.

\textbf{Intractable model:} Following the examples in \citet{liu2019fisher} and \citet{matsubara2022robust}, we consider the intractable model with density $p_\theta(y) \propto \exp(\theta_1 \tanh x_1 + \theta_2 \tanh x_2 - 0.5 \|x\|^2)$, where $\theta = (\theta_1, \theta_2)\in \R^2$, $x \in \R^3$. For $\theta = 0$, we recover the density of a Gaussian distribution $\mathcal{N}(0, I_3)$.  Given that
\begin{align*}
    s_\theta(x) = \lp\theta_1(1 - \tanh^2x_1), \theta_2(1 - \tanh^2x_2), 0\rp^T - x,
\end{align*}
we first derive
\begin{align*}
    &\| \lp\theta_1(1 - \tanh^2x_1), \theta_2(1 - \tanh^2x_2), 0\rp^T - x \| 
    \\ \leq & \| \theta \|+\| \tilde x - x \|
    \\&+ \| \tilde x - \lp\theta_1(1 - \tanh^2\tilde x_1), \theta_2(1 - \tanh^2\tilde x_2), 0\rp^T  \| 
    \\\leq &\| \theta \| + \| s_\theta(\tilde x) \| + \| \tilde x - x \|,
\end{align*}
and so $|\langle s_\theta(x), s_\theta(\tilde x) \rangle k(x, \tilde x)| $ is upper bounded by 
\begin{align*}
    & \| s_\theta(x) \| \| s_\theta(\tilde x) \| k (x, \tilde x)
    \\ \leq &\lp\| \theta \| + \| s_\theta(\tilde x) \| + \| \tilde x - x \|\rp \| s_\theta(\tilde x) \| k (x, \tilde x)
    \\\stackrel{}{\leq}&\lp\| \theta \| + \| s_\theta(\tilde x) \| + 1\rp \| s_\theta(\tilde x) \| .
\end{align*}
Secondly, $ \| s_p(\tilde x) - s_p(x) \| \leq \| \theta \|  + \|x - \tilde x\|$,
and so 
\begin{align*}
    &\lba (1 + \| x - \tilde x \|^2)^{-\frac{3}{2}}\langle s_p(\tilde x) - s_p(x), x - \tilde x\rangle \rba \\\leq &(1 + \| x - \tilde x \|^2)^{-\frac{3}{2}} \| x - \tilde x \| \lp \| \theta \|  + \|x - \tilde x\| \rp 
    \\\leq &\| \theta \|  + 1.
\end{align*}

For a fixed $\tilde x$, it thus follows that i) $|\langle s_\theta(x), s_\theta(\tilde x) \rangle k(x, \tilde x)| \leq \lp\| \theta \| + \| s_\theta(\tilde x) \| + 1\rp \| s_\theta(\tilde x) \|$, ii) $|\langle s_\theta(\tilde x), \nabla_{x} k(x, \tilde x) \rangle + \langle s_\theta(x), \nabla_{\tilde x} k(x, \tilde x) \rangle | \leq   \| \theta \|  + 1$, iii) $ \nabla_{x} \cdot \nabla_{\tilde x} k(x, \tilde x) \geq 0$.

We thus define $M_\theta(\tilde x) = \lp\| \theta \| + \| s_\theta(\tilde x) \| + 1\rp \| s_\theta(\tilde x) \| + \| \theta \|  + 1$. Note that $\E[M_{\theta}(X)] < \infty$ given that the distribution has Gaussian type tails, and so its first two moments are finite.

\textbf{Gaussian-Bernoulli Restricted Boltzmann Machine:} We consider $P$ to be a Gaussian-Bernoulli Restricted Boltzmann Machine, following related contributions in the literature \citep{liu2016kernelized, schrab2022ksd}. It is a graphical model that includes a binary hidden variable $h$, taking values in $\{1, -1\}^{d_h}$, and a continuous observable variable $x$ within $\mathbb{R}^d$. These variables are linked by the joint density function
\begin{align*}
    p(x, h) = \frac{1}{Z} \exp\lp \frac{1}{2}x^TBh + b^Tx + c^Th - \frac{1}{2} \|x\|_2^2 \rp,
\end{align*}
where $Z$ is the normalizing constant. It follows that the density $p$ of $x$ is
\begin{align*}
    p(x) = \sum_{h \in \{ - 1, 1\}^{d_h}} p(x, h).
\end{align*}
The computation of $p$ for large dimension $d_h$ is intractable; nonetheless, the score function is computable as
\begin{align*}
    s_p(x) = b - x + \frac{B}{2} \phi(\frac{B^Tx}{2} + c), \quad \phi(y) = \frac{e^{2y} - 1}{e^{2y} + 1}.
\end{align*}

We first derive 
\begin{align*}
    \left\| \frac{B}{2} \phi(\frac{B^Tx}{2} + c) - \frac{B}{2} \phi(\frac{B^T\tilde x}{2} + c)\right\| \stackrel{(i)}{\leq} \| B \|_{op}\sqrt{d_h},
\end{align*}
where (i) is obtained given that $\phi(y) \in [-1, 1]^{d_h}$ for all $y$. Thus $|\langle s_p(x), s_p(\tilde x) \rangle k(x, \tilde x)| \leq \| s_p(x) \| \|s_p(\tilde x) \| k (x, \tilde x) $ is upper bounded by $\lp \| s_p(\tilde x)\| + \| \tilde x -  x \| + \left\| \frac{B}{2} \phi(\frac{B^Tx}{2} + c) - \frac{B}{2} \phi(\frac{B^T\tilde x}{2} + c)\right\|\rp$ $\times\| s_p(\tilde x) \| k (x, \tilde x)$, which is again upper bounded by 
\begin{align*}
\lp \| s_p(\tilde x)\| +1 + \| B \|_{op}\sqrt{d_h}\rp \| s_p(\tilde x) \|.
\end{align*}

Secondly, $\| s_p(\tilde x) - s_p(x) \| \leq \| B \|_{op}\sqrt{d_h}  + \|x - \tilde x\|$, and so $ \lba (1 + \| x - \tilde x \|^2)^{-\frac{3}{2}}\langle s_p(\tilde x) - s_p(x), x - \tilde x \rangle \rba$ is upper bounded by 
\begin{align*}
    &(1 + \| x - \tilde x \|^2)^{-\frac{3}{2}} \| x - \tilde x \| \lp \| B \|_{op}\sqrt{d_h}  + \|x - \tilde x\| \rp 
    \\\leq &\| B \|_{op}\sqrt{d_h}  + 1.
\end{align*}

Lastly, we have that $\langle \nabla_{x} k(\cdot, x), \nabla_{\tilde x} k(\cdot, \tilde x) \rangle_{\mathcal{H}^d} \geq 0$. Hence it suffices to define $M_p(\tilde x) = \lp \| s_p(\tilde x)\| +1 + \| B \|_{op}\sqrt{d_h}\rp \| s_p(\tilde x) \| + \| B \|_{op}\sqrt{d_h}  + 1$. We can further upper bound $\| B \|_{op} \leq \| B \|_{fr}$, with the Frobenius norm being easily computable. Note that $\E[M_{p}(X)] < \infty$ given that the distribution has Gaussian type tails, and so its first two moments are finite.

\section{EXPERIMENTS} \label{sec:experiments}

\begin{figure*}[t!] 
  \centering
  \includegraphics[width=.9\textwidth]{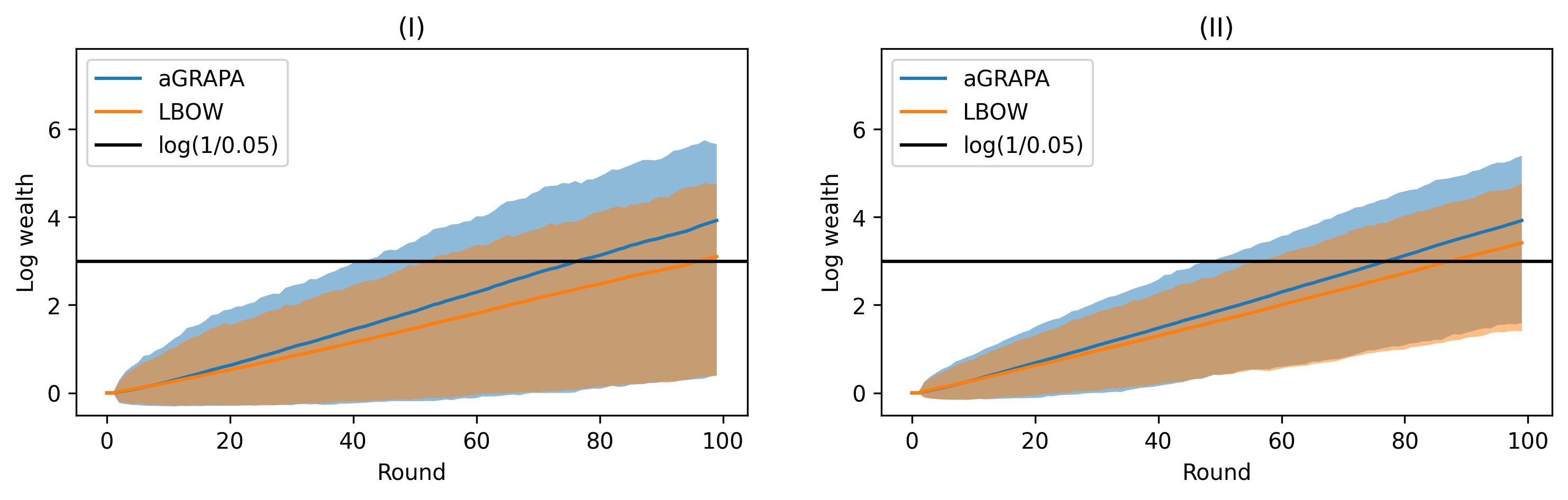}
  \caption{Average logarithmic wealth alongside $95\%$ empirical confidence intervals for $1000$ simulations under the alternative for (I) the Gaussian distribution, (II) the intractable model. We highlight the exponential growth of the wealth.} 
  \label{fig:the_four}
\end{figure*}

We consider the distributions introduced in the previous section, using the upper bounds derived therein. \textbf{Gaussian distribution:} We take $\theta_0 = 0$ under the null. \textbf{Intractable model:} We take $\theta_0 = (0, 0)$ under the null. We defer the restricted Boltzmann machine example to Appendix \ref{appendix:rbm_experiments}.  The code may be found \href{https://github.com/DMartinezT/sequential_ksd}{here}.

Figure \ref{fig:the_four} exhibits the performance of the proposed test with $\alpha = 0.05$,  under the alternatives $\theta_1 = 1$ and $\theta_1 = (1, 1)$ respectively. We emphasize the exponential growth of the wealth process. This behaviour is expected, in view of Theorem \ref{theorem:ksd_simple_power_lbow} and the fact that all these examples fulfil regularity conditions so that $\E_{H_1}[g^*(X)] > 0$;  an actual empirical verification of the lower bound derived in Theorem \ref{theorem:ksd_simple_power_lbow} is deferred to Appendix \ref{section:experiment_lower_bound}. Furthermore, we stress that aGRAPA empirically outperforms LBOW. This is due to the fact that $\lambda_{t+1}^{\text{aGRAPA}} > \lambda_{t+1}^{\text{LBOW}}$, so aGRAPA bets more aggressively. We defer the illustration of the type-I error control at the desired level $0.05$ to  Appendix~\ref{section:type1errorcontrol}, a discussion on how the tightness of the bounds $M_p$ empirically affects the power of the test to Appendix~\ref{section:tightness_bound_statistical_power}, and examples of testing composite null hypotheses to Appendix~\ref{section:composite_null_hypotheses}.   

\begin{figure*}[ht] 
    \centering
  \includegraphics[width=.9\textwidth]{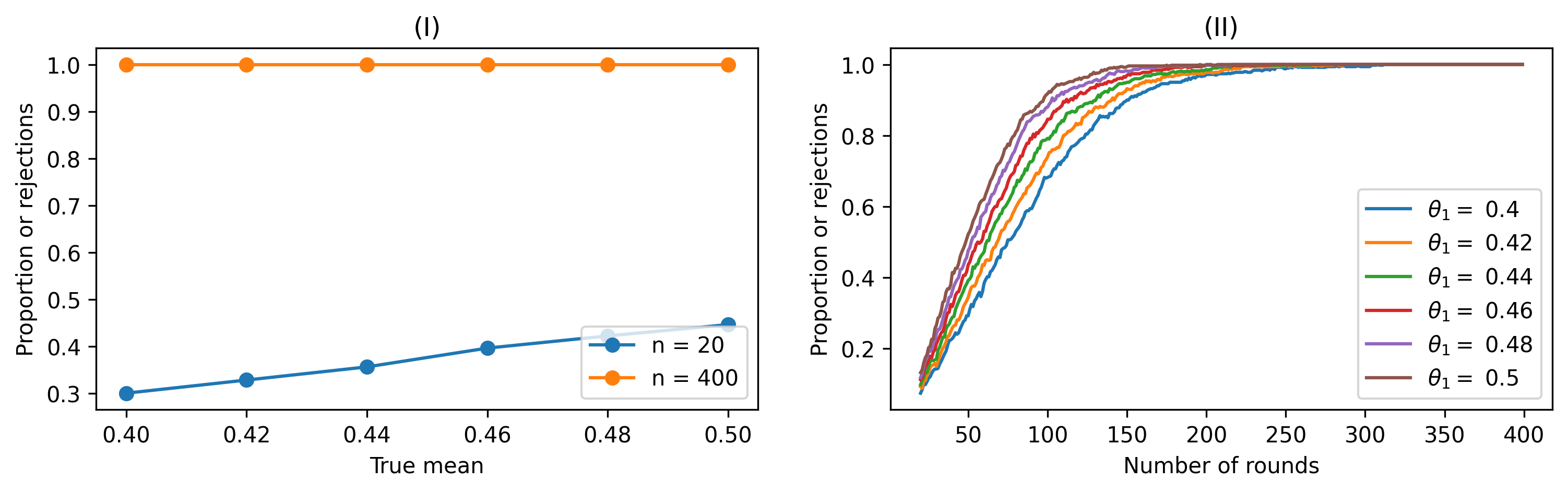}
  \caption{Proportion of rejections for the Gaussian distribution considered in Section \ref{sec:experiments} under the alternatives $\theta_1 \in \{ 0.40, 0.42, 0.44, 0.46, 0.48, 0.50 \}$ for (I) the (classical) batch setting kernelized Stein discrepancy with sample size $n$, (II) the (proposed) sequential kernelized Stein discrepancy. The (proposed) sequential test always ends up rejecting the null hypotheses, while the batch test will not do so if the original sample size is too small.}
  \label{fig:batch}
\end{figure*}

We emphasize once again that, in contrast to sequential procedures, batch setting algorithms implicitly commit to a sample size that is chosen prior to running the experiment, which may lead to substantially suboptimal choices of sample sizes. We exhibit such a phenomenon in Figure \ref{fig:batch}, which displays the proportion of rejections for the Gaussian distribution under different alternatives and illustrates the convenience of employing anytime valid tests. If using the classical kernelized Stein discrepancy test (left plot), $20$ observations lead to rejection rates lower than $0.50$, and $400$ observations show rejection rates of $1$. This implies that $20$ observations are not enough to obtain a powerful test, and probably less than $500$ observations would have sufficed to yield high power (resulting in a cheaper experiment). However, this cannot be known beforehand. In stark contrast, all the null hypotheses are eventually rejected by the proposed test (right plot), with the empirical powers being significantly high already after $150$ observations, and always being able to collect more data and keep running the experiment with anytime validity. We defer a longer comparison between the batch and sequential tests to Appendix~\ref{appendix:batch_setting}.

\section{CONCLUSION} \label{sec:conclusion}

We have developed a sequential version of the kernelized Stein discrepancy goodness-of-fit test, which gives way to goodness-of-fit tests that can handle distributions with unknown normalizing constants and can be continuously monitored and adaptively stopped. We have done so by combining tools from testing by betting with RKHS theory, while avoiding assuming uniform boundedness of the Stein reproducing kernel. We have proved the validity of the test, as well as exponential growth of the wealth process under the alternative and mild regularity conditions. Our experiments have exhibited the empirical performance of the test in a variety of scenarios.

In this contribution, we have presented a novel martingale construction that does neither exploit nor need uniform boundedness of the kernel. While the theory presented here has been developed for and motivated by the Stein kernel, such a martingale construction may be exploited by any kernel that is either unbounded or uniformly bounded by a constant that is too loose. This opens the door to develop sequential two sample and independence tests with kernels that are currently unexplored, which conforms an exciting direction of research.

\subsubsection*{Acknowledgements}
DMT gratefully acknowledges that the project that gave rise to these results received the support of a fellowship from `la Caixa' Foundation (ID 100010434). The fellowship code is LCF/BQ/EU22/11930075. AR was funded by NSF grant DMS-2310718. The authors would also like to thank the anonymous reviewers for their valuable comments and suggestions.

\bibliographystyle{apalike}  
\bibliography{references}

 \clearpage
\appendix
 \onecolumn
\section{ADDITIONAL EXPERIMENTS}

\subsection{Logarithmic wealth process of a Gaussian-Bernoulli restricted Boltzmann machine} \label{appendix:rbm_experiments}

In Section \ref{sec:experiments}, we displayed the performance of the proposed test for the Gaussian and intractable models. The performance of the proposed test for the remaining model introduced in Section \ref{sec:derivation}, a Gaussian-Bernoulli restricted Boltzmann machine, is now exhibited in Figure \ref{fig:rbms}. In particular, we take $d_h = 10$, and $d = 50$. We sample from it using Gibbs sampling with a burn-in of 1000 iterations. Under the null, we take $b = 0$, $c = 0$, and matrix $B$ such that $B_{ij}$ is one if visible node $i$ is connected to hidden node $j$, and zero otherwise, with each hidden node connected to five visible nodes. We consider two different alternatives. For one of them, each entry of $B$ is shifted by $0.5$. For the other one, $b = 1$.

\begin{figure*}[ht] 
  \centering
  \includegraphics[width=\textwidth]{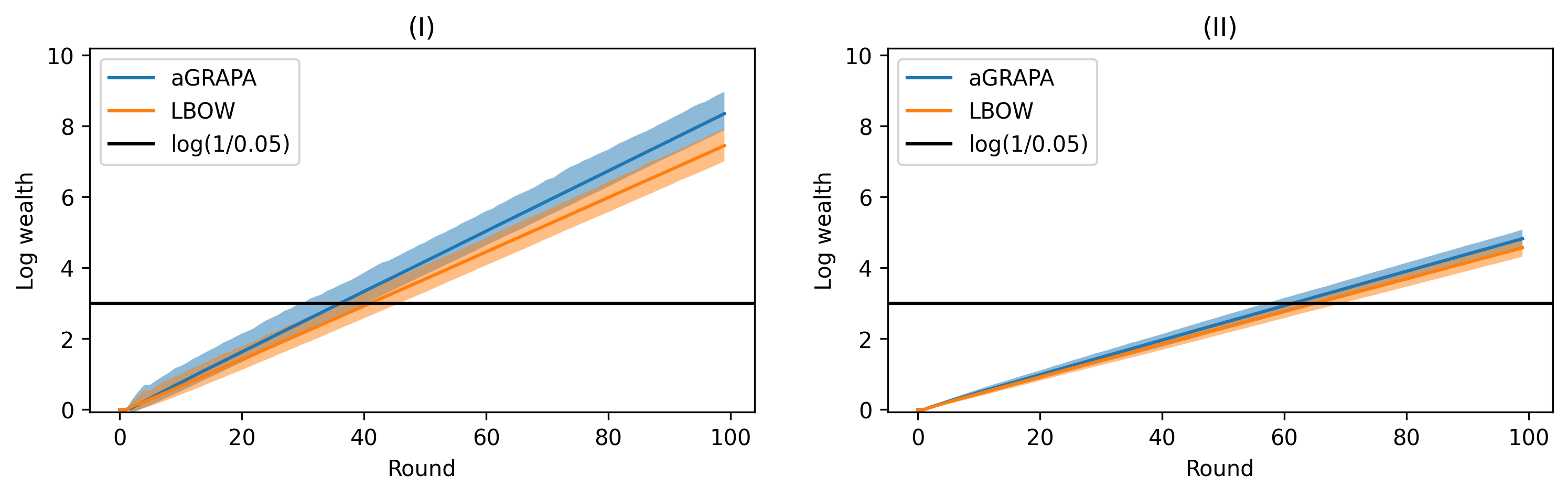}
  \caption{Average logarithmic wealth alongside $95\%$ empirical confidence intervals for $1000$ simulations under the alternative for (I) the restricted Boltzmann machine with shifted $B$, (II) the restricted Boltzmann machine with bias $b=1$. We highlight once again the exponential growth of the wealth processes.} 
  \label{fig:rbms}
\end{figure*}

\subsection{Type-I error control} \label{section:type1errorcontrol}

In Section~\ref{sec:experiments}, we illustrated the empirical power of the proposed test. We devote this appendix to displaying the empirical type-I error control of the proposed test.  Figure~\ref{fig:validity_raw_the_three_plots} exhibits the wealth processes for the examples considered in Section~\ref{sec:experiments} under the null. We highlight that the wealth processes do not cross the threshold $1/0.05$, and hence the nulls are not rejected. 

\begin{figure}[b!] 
\centering
  \includegraphics[width=\textwidth]{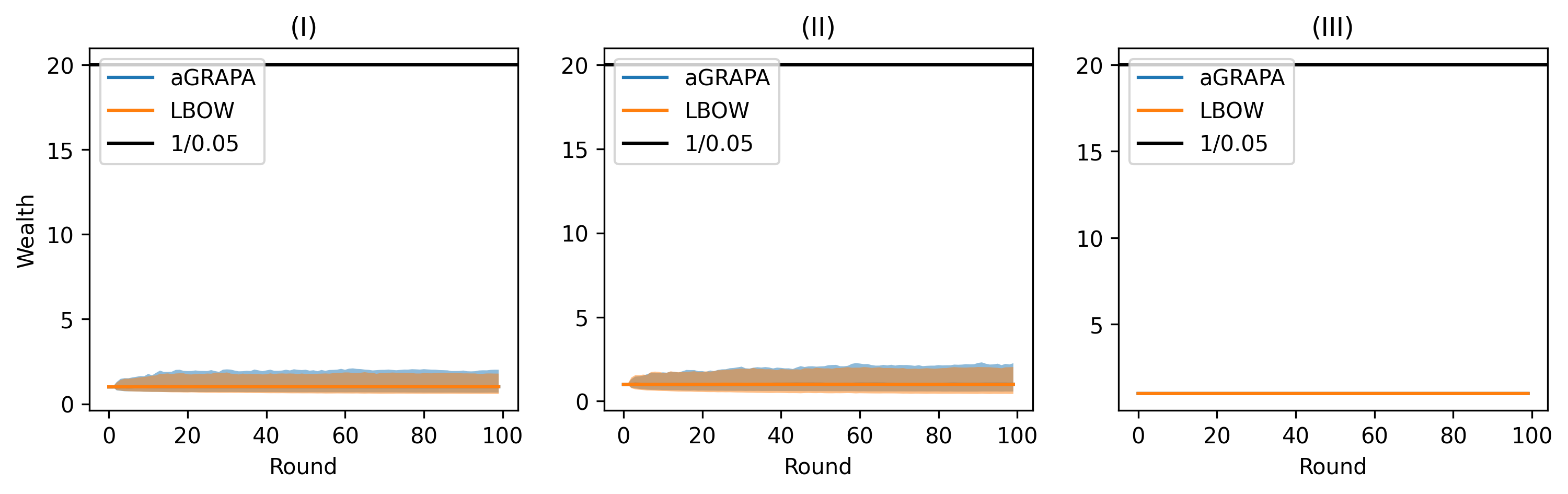}
  \caption{Average wealth alongside $95\%$ empirical confidence intervals for $1000$ simulations under the null for (I) the Gaussian distribution, (II) the intractable model, (III) the restricted Boltzmann machine. We emphasize that the wealth processes do not cross the threshold $1/0.05$, and hence the nulls are not rejected, showing the empirical type-I error control.}
  \label{fig:validity_raw_the_three_plots}
\end{figure}

We note that Figure~\ref{fig:validity_raw_the_three_plots} exhibits the anytime validity of the tests until round $100$ (the same number of rounds as considered in Section~\ref{sec:experiments}). However, we can verify anytime validity without committing to a sample size for those simulations with known normalizing constants.  More specifically, let $P$ be the distribution under the null, and $Q$ any other distribution such that $dP/dQ$ is known and we can sample from $Q$. Let $\tau$ denote the random stopping time. Since $\mathbb{P}_P(\tau < \infty) = E_{P}[1(\tau < \infty)] = E_{Q}[dP/dQ 1(\tau < \infty)]$, instead of sampling from $P$ and calculating $\mathbb{P}_P(\tau < \infty)$ by Monte-Carlo, we can instead sample from $Q$ and calculate $E_{Q}[dP/dQ 1(\tau < \infty)]$ by Monte-Carlo. For $Q\neq P$, the procedure will stop relatively fast (since the wealth grows exponentially fast under the alternative), and we can average these likelihood ratios at the stopping time to estimate the type 1 error. For simulations in which we know the normalizing constants (even though we may not want to use them in order to test our methods), this importance sampling technique for estimating the type-1 error works well. For the Gaussian case considered in Section~\ref{sec:experiments} (i.e., $P$ is a standard Gaussian distribution), taking $Q$ to be a Gaussian distribution with unit variance centered at $0.5$ leads to an estimated  $\mathbb{P}_P(\tau < \infty)$ of $0.0006$ for $\alpha = 0.1$ (showing, again, the empirical anytime validity of the proposed test). 

\subsection{Comparison to batch setting} \label{appendix:batch_setting}

We illustrated in Section \ref{sec:experiments} some of the advantages of the sequential test over the batch test. We emphasize that the proposed test is anytime valid, counting with the subsequent substantial advantages that have been discussed in the main body of the work. Thus, it is expected to be outperformed in the batch setting by algorithms that are tailored to fixed sample sizes. Nonetheless, it is of interest to explore how the proposed test compares with the more classical kernel Stein discrepancy test. We present in this appendix the empirical performance of both the sequential and fixed sample size algorithms. 

Figure \ref{fig:batch_appendix} exhibits the proportion of rejections for the Gaussian distribution considered in Section \ref{sec:experiments} with different alternatives. As expected, the classical batch setting test outperforms the proposed test for fixed sample sizes, with the proportions of rejections of the former being always larger than those of the latter. Nonetheless, the figure illustrates the convenience of employing anytime valid tests. If using the classical kernel Stein discrepancy test, $50$ observations suffice to reject the null hypothesis for every sample in the right plot, but they only allow to reject for around $80\%$ of the samples in the left plot. The null hypothesis cannot be ever rejected for the samples encompassed in the remaining $20\%$: based on the interpretation of p-values, it is not possible to keep gathering evidence to rerun the test. In stark contrast, all the null hypotheses are eventually rejected by the proposed test, being always able to collect more data and keep running the experiment with anytime validity.

\begin{figure}[ht] 
  \includegraphics[width=\textwidth]{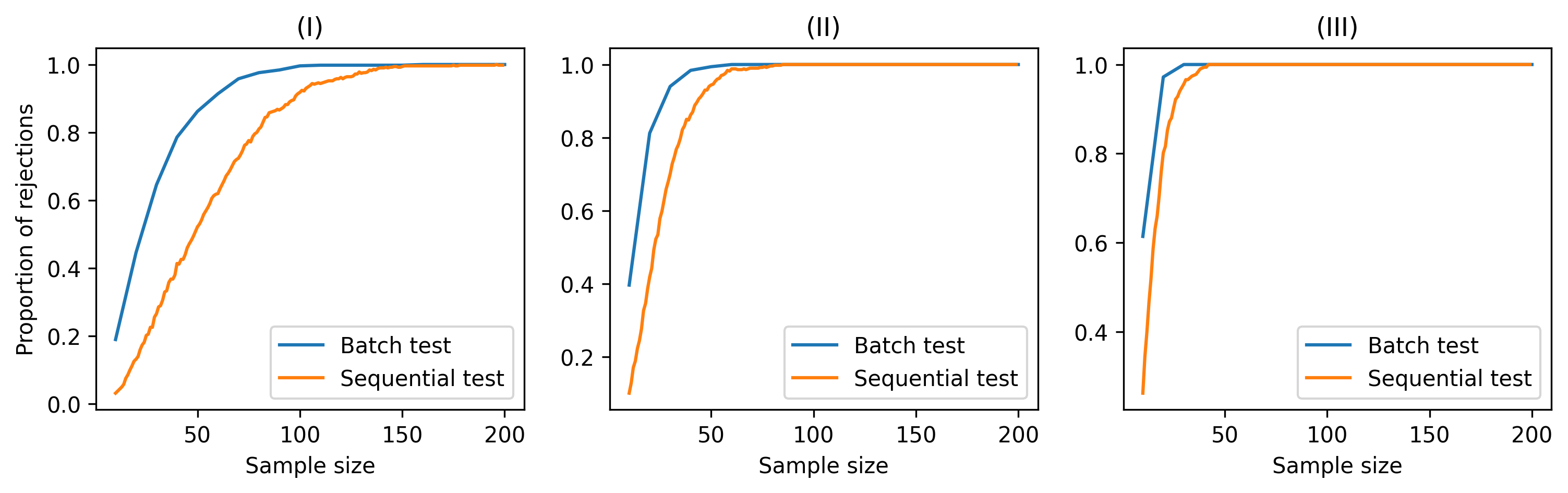}
  \caption{Proportion of rejections for $500$ simulations for the Gaussian distribution considered in Section \ref{sec:experiments} under three different alternatives with (I) $\theta = 0.5$, (II) $\theta = 0.75$, (III) $\theta = 1$. We emphasize that the (proposed) sequential test always ends up rejecting the null hypotheses, while the batch test will not do so if the original sample size is too small.}
  \label{fig:batch_appendix}
\end{figure}

We would like to emphasize the lack of anytime validity of the classical, batch setting test. Figure \ref{fig:non_validity} displays the proportion of rejections for both the proposed sequential test, and the batch test run sequentially (once every time a new observation is observed). Not surprisingly, the batch test rejection rates go well above the desired $\alpha = 0.05$ type-I error. This example illustrates the lack of anytime validity of classical test.

\begin{figure}[ht] 
  \includegraphics[width=\textwidth]{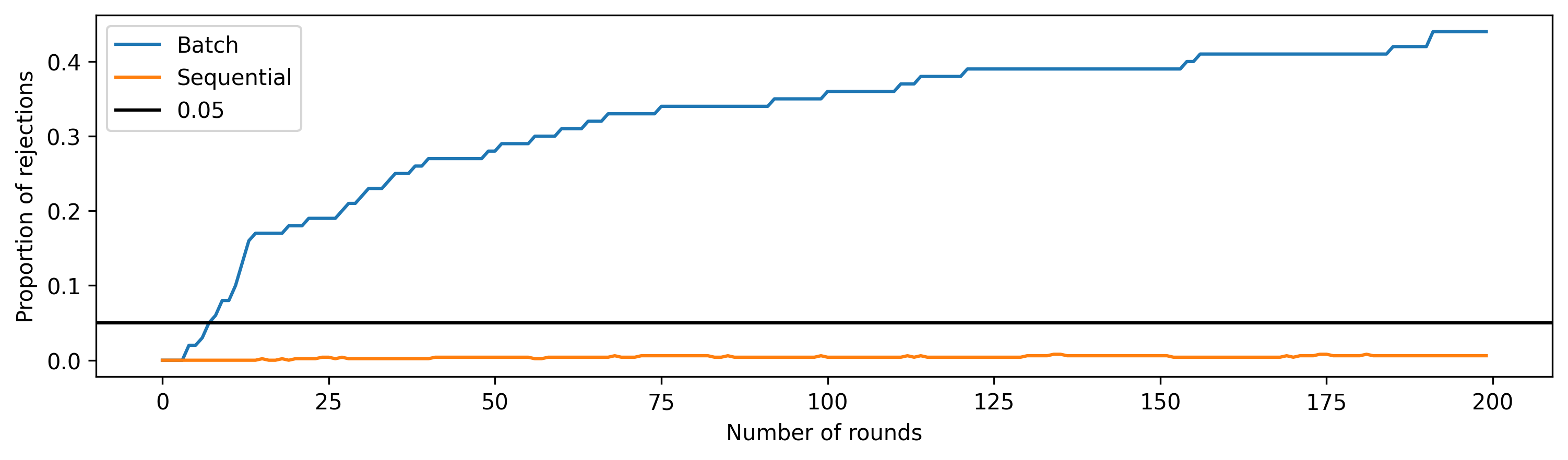}
  \caption{Proportion of rejections for $100$ simulations for the Gaussian distribution considered in Section \ref{sec:experiments} under the null for the proposed sequential test, and the batch test run sequentially (once every time a new observation is observed). While the sequential test preserves anytime valid guarantees, the `sequentialized' batch test lacks type-I error control.}
  \label{fig:non_validity}
\end{figure}

\subsection{Composite null hypotheses} \label{section:composite_null_hypotheses}

Section~\ref{section:composite_null_hypothesis} presented the extension of the proposed test to composite null hypotheses. We devote this appendix to illustrating the validity and power of such an extension. Let the null hypothesis be $H_0: Q \in \mathcal{P}$ against $H_1: Q \not\in\mathcal{P}$. Throughout, the null hypotheses will be combinations of the Gaussian-Bernoulli restricted Boltzmann machines (RBMs) from Section~\ref{sec:derivation}. 

In particular (and analogously to Appendix~\ref{appendix:rbm_experiments}), we take $d_h = 10$, and $d = 50$. We sample from such RBMs using Gibbs sampling with a burn-in of 1000 iterations. Throughout, $Q$ is taken as an RBM with $b = 0$, $c = 0$, and matrix $B$ such that $B_{ij}$ is one if visible node $i$ is connected to hidden node $j$, and zero otherwise, with each hidden node connected to five visible nodes. For two different experiments, we consider the two null hypotheses $\mathcal{P} = \{ Q, Q_{a,1} \}$ and $\mathcal{P} = \{ Q_{a,1}, Q_{a,2} \}$, where $Q_{a,1}$ is an RBM with each entry of $B$ is shifted by $0.5$ and $Q_{a,2}$ is taken with $b = 1$. Figure~\ref{fig:RBM_two_different_composite} displays the logarithmic wealth processes for these two scenarios. Note that the logarithmic wealth process never crosses the threshold $\log(1/0.05)$ in the first case (showing the empirical anytime validity of the test), while the latter case shows comparable power to those experiments presented in Section~\ref{appendix:rbm_experiments}.

\begin{figure}[ht] 
  \includegraphics[width=\textwidth]{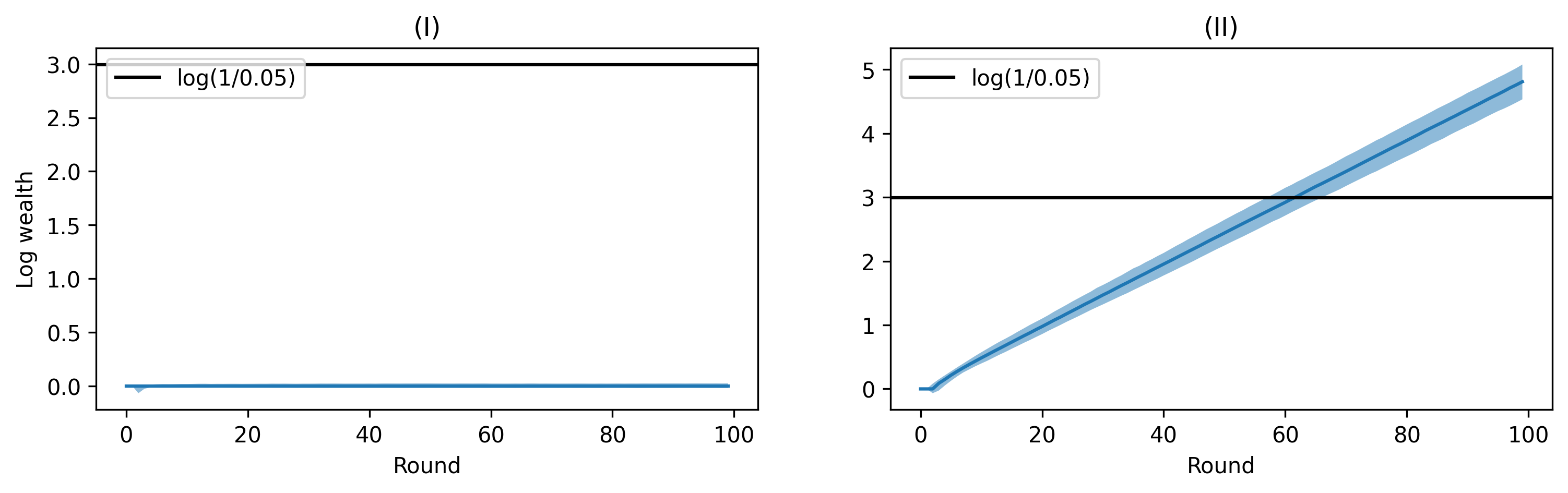}
  \caption{Average logarithmic wealth for $1000$ simulations for the composite null hypothesis (I) $\mathcal{P} = \{ Q, Q_{a,1} \}$, (II) $\mathcal{P} = \{ Q_{a,1}, Q_{a,2} \}$, where $Q$ is the RBM defined in Section~\ref{sec:derivation} with $b = 0$, $c = 0$, and matrix $B$ such that $B_{ij}$ is one if visible node $i$ is connected to hidden node $j$, and zero otherwise, with each hidden node connected to five visible nodes; $Q_{a,1}$ is the same RBM with each entry of $B$ is shifted by $0.5$, and $Q_{a,2}$ is taken with $b = 1$.}
  \label{fig:RBM_two_different_composite}
\end{figure}

\subsection{On the tightness of the bounds and statistical power} \label{section:tightness_bound_statistical_power}

In Section~\ref{sec:derivation}, we carefully exhibited how to derive sensible bounds $M_p$ for different families of distributions. Here, we explore how the tightness of the bounds affects the statistical power of the test. Figure~\ref{fig:different_bounds} displays the wealth processes for the Gaussian distribution from Section~\ref{sec:derivation} under three different alternatives, using the original bound obtained in Section~\ref{sec:derivation} alongside looser bounds.  We highlight that all logarithmic wealth processes are linear (i.e., the wealth growth is exponential), but looser bounds lead to smaller slopes.

\begin{figure}[ht] 
\centering
  \includegraphics[width=\textwidth]{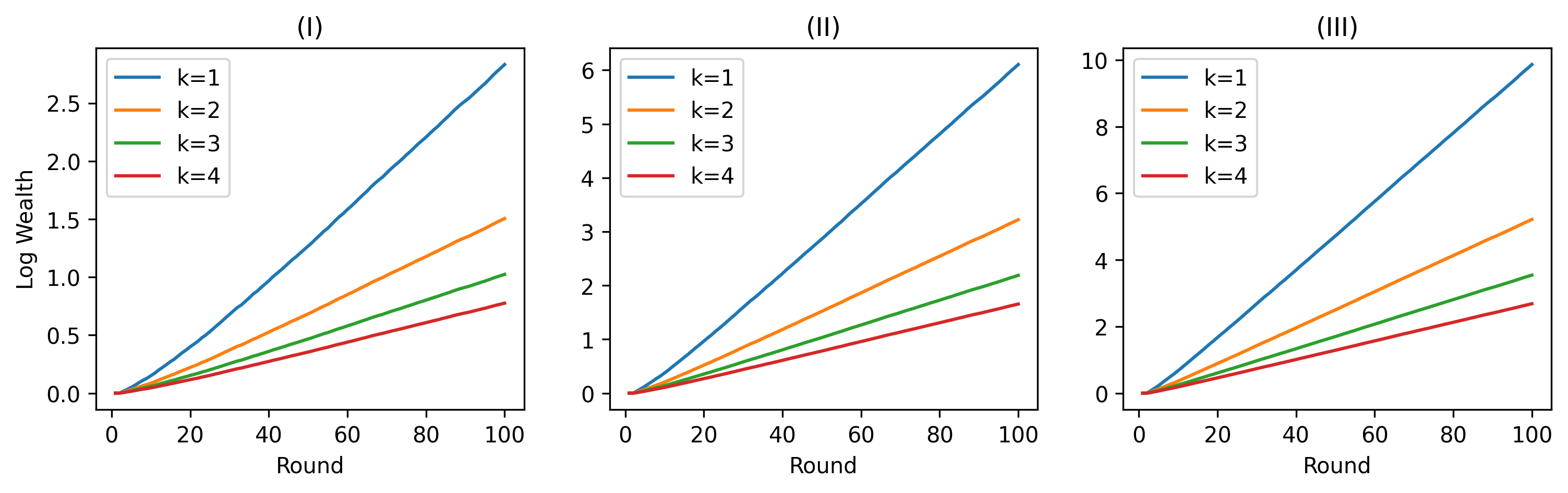}
  \caption{Average logarithmic wealth for $1000$ simulations for the Gaussian distribution when using bounds of the form $k \cdot M(X)$, where $M(X)$ is the original bound derived in Section~\ref{sec:derivation} and $k \in \{1, 2, 3, 4\}$, under the alternatives (I) $\theta_1 = 0.5$, (II) $\theta_1 = 0.75$, (III) $\theta_1 = 1$. It can be clearly seen that all the logarithmic wealth processes are linear (i.e., the wealth growth is exponential), but higher multiplicative factors $k$ (i.e., looser bounds) have smaller slopes.}
  \label{fig:different_bounds}
\end{figure}

\subsection{Empirical Verification of the Lower Bound of the Wealth} \label{section:experiment_lower_bound}

In Section \ref{sec:main_results}, we stressed that the stopping time $\tau$ of the proposed test  (i.e. $\tau$ is the smallest $t$ verifying $\wealth_t \geq 1/\alpha$) is roughly upper bounded by $\log(1/\alpha) / r^*$ (where $r^*$ is defined as in Theorem \ref{theorem:ksd_simple_power_lbow}). We devote this section to exhibit the empirical validity of this claim. Note that, while $r^*$ is unknown in practice, it can be easily estimated via a Monte Carlo approach when we have access to the ground truth distribution (as it is the case in our experimental settings), given that it only depends on the first and second moments of the $h_p$ and $M_p$ evaluations.

Figure \ref{fig:lower_bound} displays the stopping times $\tau$ of the proposed test  for the Gaussian setting (presented in Section \ref{sec:specific_derivations} and Section \ref{sec:experiments}) as a function of $r^*$, both for the LBOW and aGRAPA strategies. We take $\theta_0 = 0$ under the null, and a range of $\theta_1$ under different alternatives. Intuitively, the larger the distance between the means  $\theta_0$ and $\theta_1$ is, the larger the theoretical quantities $r^*$ become. Thus, this setting provides a range of $r^*$ for which we can study the empirical stopping times $\tau$. Figure \ref{fig:lower_bound} shows that the average stopping time curve, as well as its empirical $95\%$ confidence interval, are empirically dominated by the upper bound $\log(1/\alpha)/r^*$ 
 (and once more, we note that the aGRAPA strategy empirically outperforms the LBOW strategy). 

\begin{figure}[ht] 
\centering
  \includegraphics[width=.8\textwidth]{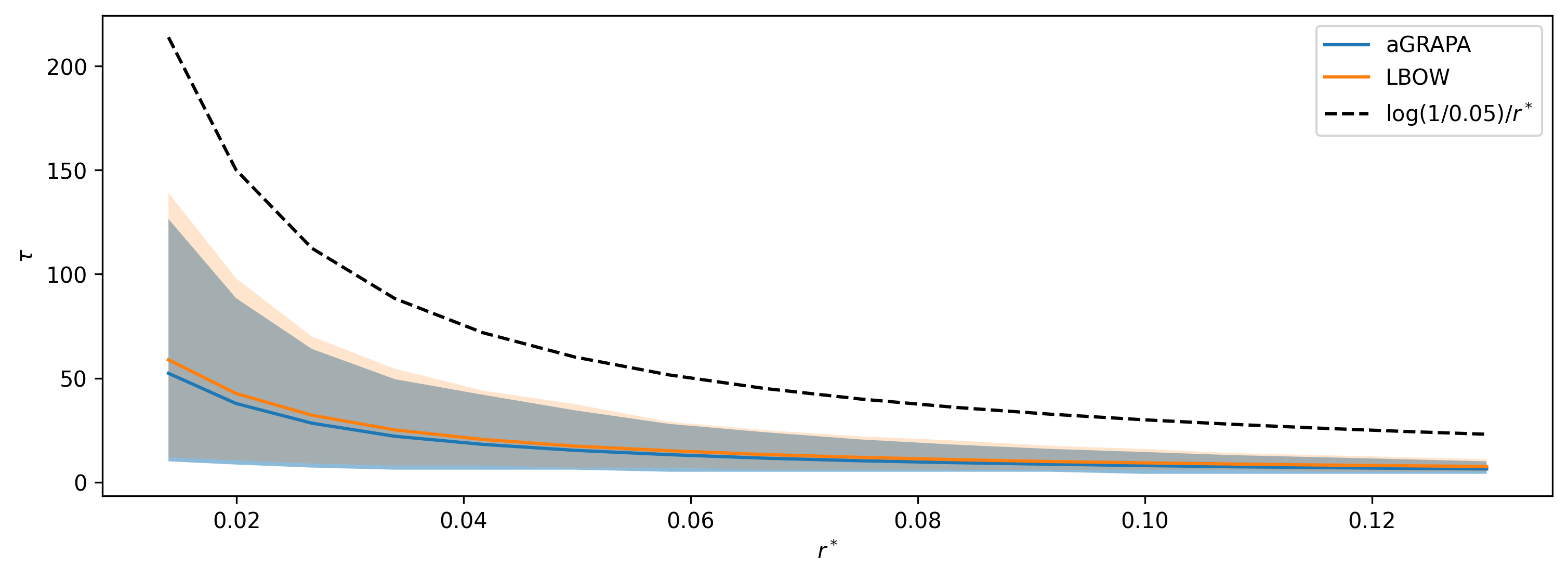}
  \caption{Average empirical stopping values $\tau$, alongside their $95\%$ empirical confidence intervals for $1000$ simulations, as a function of $r^*$ (defined in Theorem \ref{theorem:ksd_simple_power_lbow}) in the Gaussian setting, with $\alpha = 0.05$, across $500$ simulations, for the LBOW and aGRAPA strategies. We emphasize that all the stopping times are dominated by the approximate upper bound $log(1/\alpha) / r^*$.}
  \label{fig:lower_bound}
\end{figure}

\clearpage
\section{Comparison to related work} \label{appendix:comparison}

We present in this appendix a comprehensive comparison of our contribution and that of \citet{zhou2024sequential}, which also proposed a sequential KSD goodness-of-fit test that is based on a supermartingale construction and Ville's inequality.

The primary distinction between the contributions lies in the uniform boundedness assumption, which is adopted in \citet{zhou2024sequential} but deliberately avoided in our work. Specifically, their work assumed that the scores are uniformly bounded by 1 \citep[Assumption 2.4]{zhou2024sequential}. We would like to emphasize that this is, in principle, a rather restrictive assumption. For instance, even the simple case of a Gaussian distribution with an IMQ kernel does not satisfy this assumption. Thus, our method is strictly more general and can accommodate more scenarios, e.g., the three experiments presented in Section~\ref{sec:experiments} could not have been addressed using their procedure.

In principle, there could be other alternatives to work around unbounded score functions. For instance, one could work with tilted base kernels that make the Stein kernel bounded even if the score function is not bounded   \citep[Theorem 7]{barp2024targeted}. Nonetheless, current results such as \citet[Theorem 7]{barp2024targeted} still require bounds to `tilt' the base kernel, but do not provide such bounds (the authors' goal is solely to prove that boundedness of the Stein kernel may be assumed without loss of generality in their setting, without further interest in the bound itself).  One of the main contributions of our paper is to show how to obtain tight workable bounds, and illustrate that such derivations yield powerful tests. In order to potentially use \citet[Theorem 7]{barp2024targeted} to obtain sequential tests, one would have to carefully study the actual bounds derived from such a theorem. This would eventually lead to an analysis very similar to the one presented throughout this contribution.

The lack of uniform boundedness also conveys a number of deep theoretical challenges. The theoretical foundation of \citet{zhou2024sequential} is analogous to the one presented in \citet{podkopaev2023sequential}. See e.g. \citet[Theorem 2.5]{zhou2024sequential} and \citet[Theorem 2.4]{podkopaev2023sequential}, alongside their proofs. However, the lack of uniform boundedness makes such proofs break in (at least) two different places: (i) after applying the Cauchy-Schwarz inequality, see \citet[Equation (47)]{zhou2024sequential}, (ii) the analysis of the Online Newton Step (ONS) strategy for selecting betting fractions. In order to circumvent the above challenges, this contribution replaces the ONS strategy for aGRAPA/LBOW strategies and presents novel theoretical guarantees, which are exhibited in Appendix~\ref{appendix:auxiliary_results}.\footnote{Some of the arguments presented in Appendix~\ref{appendix:auxiliary_results} may be of independent interest. To the best of our knowledge, no other contribution in the `testing by betting' literature exploits outcomes that are lower bounded but not upper bounded.} Lastly, we have experimentally found that the aGRAPA and LBOW strategies proposed in our contribution yield substantially more powerful tests than the ONS strategy proposed in \citet{zhou2024sequential} in all the cases considered. As exhibited in Figure~\ref{fig:four_ons},  the average log wealth at round $100$ for aGRAPA roughly doubles the log wealth achieved by the ONS strategy for all the experiments.

\begin{figure}[ht] 
\centering
  \includegraphics[width=\textwidth]{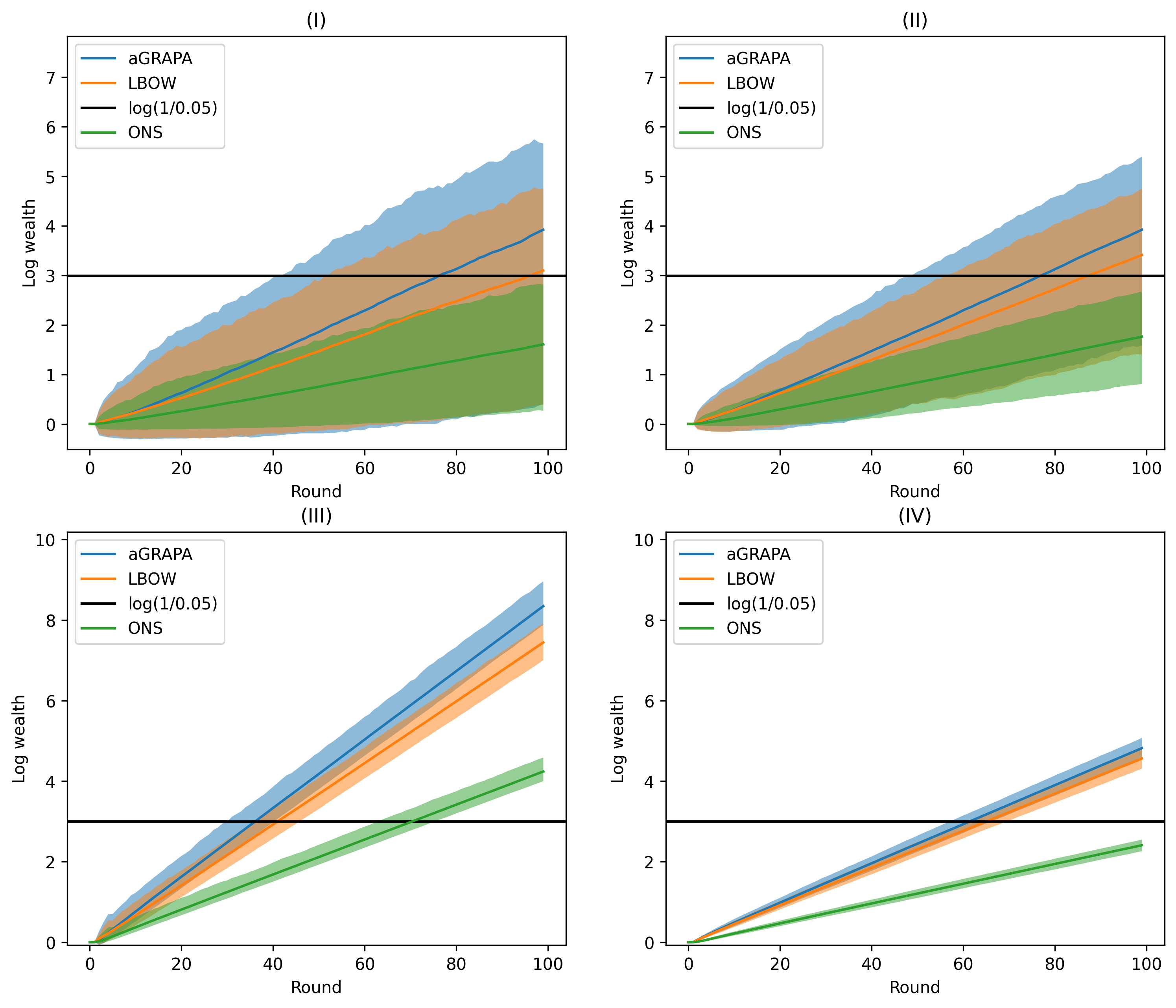}
  \caption{Average logarithmic wealth alongside $95\%$ empirical confidence intervals for $1000$ simulations under the alternative for (I) the Gaussian distribution, (II) the intractable model, (III) the restricted Boltzmann machine with shifted $B$, (IV) the restricted Boltzmann machine with bias $b=1$. We highlight that the logarithmic wealth achieved by ONS is roughly doubled by that obtained by the aGRAPA strategy.}
  \label{fig:four_ons}
\end{figure}
\clearpage
\section{PROOFS} \label{appendix:proofs}

\subsection{Notation}

Throughout, we denote 
\begin{align*}
    g_i(x) &= \frac{1}{\frac{1}{t-1}\sum_{k=1}^{i-1} M_{p}(X_k)}\left(\frac{1}{i-1} \sum_{k = 1}^{i-1} h_p(X_k, x)\right), \quad g^*(x) = \frac{1}{\E[M_p(X)]}\E[ h_p(X, x)],
    \\
    f_i(x) &= \frac{1}{i-1} \sum_{k = 1}^{i-1} h_p(X_k, x), \quad f^*(x) = \E[ h_p(X, x)],
   \\ M^i_p &= \frac{1}{i-1}\sum_{k=1}^{t-1} M_p(X_k),\quad M_p = E[M_p(X)].
\end{align*}

\subsection{Preliminary Results}

For completeness, we enunciate the following theorems, which will be exploited in subsequent proofs. 

\begin{theorem} [SLLN for Banach-valued random variables. {\citet[Theorem 2.4]{bosq2000linear}}.] \label{theorem:slln_banach}
Let $B$ be a separable Banach space with norm $\|\cdot\|_B$. Let $(\chi_t)_{t \geq 1}$ be a sequence of i.i.d. integrable $B$-valued random variables. Then
\begin{align*}
    \| \frac{1}{t} \sum_{i = 1}^t \chi_i - \E[\chi] \|_B \stackrel{a.s.}{\to} 0.
\end{align*}
\end{theorem}

\begin{theorem} [Ville's inequality.] Let $(S_t)_{t \geq 0}$ be a nonnegative supermartingale process adapted to a filtration $\{ \F_t \geq 0 \}$. It holds that
\begin{align*}
    \Pb(\exists t \geq 1 : S_t \geq x) \leq \frac{\E[S_0]}{x}
\end{align*}
for any $x > 0$.
\end{theorem}

\subsection{Auxiliary Results} \label{appendix:auxiliary_results}

We present here two results that Theorem \ref{theorem:ksd_simple_power_lbow}  builds on. 
 
\begin{proposition} \label{proposition:first_order}
If $\E[\sqrt{h_p(X, X)}] < \infty$ and $M_p < \infty$, then
\begin{align*}
    \frac{1}{t}\sum_{i=1}^t \left| g_i(X_i) - g^*(X_i) \right| \stackrel{a.s.}{\to} 0, \quad \frac{1}{t}\sum_{i=1}^t g_i(X_i) \stackrel{a.s.}{\to} \E[g^*(X)].
\end{align*} 
\end{proposition}

\begin{proof}

Without loss of generality, we can assume that $M_p > 0$. We first highlight that, based on $0 < \E[\sqrt{h_p(X, X)}] < \infty$, 
\begin{itemize}
    \item $\E[\lba f^*(X) \rba]$ is finite, as $\E[\lba f^*(X) \rba] = \E_{X, X'}[\lba h_p(X', X)\rba] \leq \E_{X, X'}\left[ \left\| \xi_p(\cdot, X)  \right\|_{\mathcal{H}^d}  \left\| \xi_p(\cdot, X')  \right\|_{\mathcal{H}^d} \right] = \E_{X}\left[ \left\| \xi_p(\cdot, X)  \right\|_{\mathcal{H}^d} \right]^2 = \E[\sqrt{h_p(X, X)}]^2 < \infty$,
    \item $\E\lb\left\| \xi_p(\cdot, X)  \right\|_{\mathcal{H}^d} \rb = \E[\sqrt{h_p(X, X)}]$ is finite.
    \item $\E[\lba g^*(X)\rba]$ is finite, as $g^*(x) = f^*(x) / M_p$, $\E[\lba f^*(X) \rba]$ is also finite and $M_p > 0$. 
\end{itemize}

We now note that $\frac{1}{t}\sum_{i=1}^t \left| g_i(X_i) - g^*(X_i) \right| \stackrel{a.s.}{\to} 0$ implies $\frac{1}{t}\sum_{i=1}^t g_i(X_i) \stackrel{a.s.}{\to} \E[g^*(X)]$, i.e., $\frac{1}{t}\sum_{i=1}^t g_i(X_i) - \E[g^*(X)] \stackrel{a.s.}{\to} 0$. To see this, we decompose
\begin{align*}
    \frac{1}{t}\sum_{i=1}^t g_i(X_i) - \E[g^*(X)] = \lp \frac{1}{t}\sum_{i=1}^t g_i(X_i) - \frac{1}{t}\sum_{i=1}^t g^*(X_i) \rp + \lp\frac{1}{t}\sum_{i=1}^t g^*(X_i) - \E[g^*(X)] \rp,
\end{align*}
and we highlight that the second term converges to zero almost surely in view of the strong law of large numbers (SLLN) and $\E[\lba g^*(X)\rba] < \infty$. Further, $\frac{1}{t}\sum_{i=1}^t \left| g_i(X_i) - g^*(X_i) \right| \stackrel{a.s.}{\to} 0$ implies $\frac{1}{t}\sum_{i=1}^t g_i(X_i) - \frac{1}{t}\sum_{i=1}^t g^*(X_i) \stackrel{a.s.}{\to} 0.$

Hence, it remains to prove that
\begin{align*}
    \frac{1}{t}\sum_{i=1}^t \left| g_i(X_i) - g^*(X_i) \right| \stackrel{a.s.}{\to} 0.
\end{align*}

That is, for any given $\epsilon > 0$ and $\delta > 0$, there exists $N \in \mathbb{N}$ such that
\begin{align*}
    \Pb \lp \sup_{t \geq N} \frac{1}{t}\sum_{i=1}^t \left|g_i(X_i) - g^*(X_i) \right| > \epsilon \rp \leq \delta.
\end{align*}

In order to prove the result, we are going to decompose $\frac{1}{t}\sum_{i=1}^t \left|g_i(X_i) - g^*(X_i) \right|$, and then combine the scalar-valued SLLN and the Banach space-valued SLLN. We will finish the proof by taking a union bound over the different terms.

\textit{Introducing the SLLN terms and the probabilistic bounds:} By the scalar valued SLLN and the finiteness of $M_p > 0$, $\E[\lba f^*(X) \rba]$, and $\E\lb\left\| \xi_p(\cdot, X)  \right\|_{\mathcal{H}^d} \rb$, we have that
\begin{itemize}
    \item $M_p^i \stackrel{a.s.}{\to} M_p$, and so $\frac{    1}{M_p^i} \stackrel{a.s.}{\to} \frac{1}{M_p}$ (given that $M_p \neq 0$), 
    \item $\frac{1}{t}\sum_{i=1}^{t} \left|f^*(X_i)\right| \stackrel{a.s.}{\to} \E[\lba f^*(X) \rba]$,
    \item $\frac{1}{t}\sum_{i=1}^{t}  \left\| \xi_p(\cdot, X_i)  \right\|_{\mathcal{H}^d} \stackrel{a.s.}{\to} \E\lb\left\| \xi_p(\cdot, X)  \right\|_{\mathcal{H}^d} \rb$.
\end{itemize}

From the Banach space-valued SLLN (Theorem \ref{theorem:slln_banach}) and the finiteness of $\E\lb\left\| \xi_p(\cdot, X)  \right\|_{\mathcal{H}^d} \rb$, it also follows that
\begin{itemize}
    \item $\left\| \frac{1}{i-1} \sum_{k = 1}^{i-1} \xi_p (\cdot, X_k) - \E[\xi_p (\cdot, X)] \right\|_{\mathcal{H}^d} \cas 0$.
\end{itemize}

Hence there exist $B > 0$ and $N_1 \in \Nb$ such that
\begin{align*}
    &\Pb \lp \sup_{t \geq N_1} \left| \frac{1}{M_p^i} \right| > B \rp \leq \frac{\delta}{10},
    \\
    &\Pb \lp \sup_{t \geq N_1} \frac{1}{t}\sum_{i=1}^{t} \left|f^*(X_i)\right| > B \rp \leq \frac{\delta}{10},
    \\
    &\Pb \lp \sup_{t \geq N_1} \frac{1}{t}\sum_{i=1}^{t}  \left\| \xi_p(\cdot, X_i)  \right\|_{\mathcal{H}^d} > B \rp \leq \frac{\delta}{10},
    \\
    &\Pb \lp \sup_{t \geq N_1} \left| \frac{1}{M_p^i} - \frac{1}{M_p} \right| > \frac{\epsilon}{4B} \rp \leq \frac{\delta}{10},
    \\&\Pb \lp \sup_{i \geq N_1} \left\| \frac{1}{i-1} \sum_{k = 1}^{i-1} \xi_p (\cdot, X_k) - \E[\xi_p (\cdot, X)] \right\|_{\mathcal{H}^d} > \frac{\epsilon}{4B^2} \rp \leq \frac{\delta}{10}.
\end{align*}

Note that $\Pb \lp \sup_{t \geq N_1} \frac{1}{t}\sum_{i=1}^{t}  \left\| \xi_p(\cdot, X_i)  \right\|_{\mathcal{H}^d} > B \rp \leq \frac{\delta}{10}$ and $\Pb \lp \sup_{t \geq N_1} \frac{1}{t}\sum_{i=1}^{t} \left|f^*(X_i)\right| > B \rp \leq \frac{\delta}{10}$ imply that $\Pb \lp \sup_{t \geq 2N_1} \frac{1}{t-N_1}\sum_{i=N_1}^{t}  \left\| \xi_p(\cdot, X_i)  \right\|_{\mathcal{H}^d} > B \rp \leq \frac{\delta}{10}$ and $\Pb \lp \sup_{t \geq 2N_1} \frac{1}{t-N_1}\sum_{i=N_1}^{t} \left|f^*(X_i)\right| > B \rp \leq \frac{\delta}{10}$, given that the data are iid.

\textit{Combining the probabilistic bounds:} We start by noting that
\begin{align*}
     \frac{1}{t}\sum_{i=1}^t \left|g_i(X_i) - g^*(X_i) \right| &= \frac{1}{t}\sum_{i=1}^{N_1-1} \left| g_i(X_i) - g^*(X_i) \right| + \frac{1}{t}\sum_{i=N_1}^{t} \left| g_i(X_i) - g^*(X_i) \right|.
\end{align*}
Consider the sequence of random variables $Y_t = \frac{1}{t}\left(\sum_{i=1}^{N_1-1} \left| g_i(X_i) - g^*(X_i) \right|\right)$. Clearly, $Y_t \stackrel{a.s.}{\to} 0$. Thus there exists $N_2$ such that $\sup_{t \geq N_2} Y_t \leq \frac{\epsilon}{2}$ with probability at least $1 - \delta /2$. 

Moreover,
\begin{align*}
    \frac{1}{t}\sum_{i=N_1}^{t} \left| g_i(X_i) - g^*(X_i) \right| &= \frac{1}{t}\sum_{i=N_1}^{t} \left| \frac{1}{M_p^i} f_i(X_i) - \frac{1}{M_p}f^*(X_i) \right|
    \\ &= \frac{1}{t}\sum_{i=N_1}^{t} \left| \frac{1}{M_p^i} f_i(X_i) - \frac{1}{M_p^i} f^*(X_i) + \frac{1}{M_p^i} f^*(X_i) - \frac{1}{M_p}f^*(X_i) \right|
    \\ &\leq \underbrace{\frac{1}{t}\sum_{i=N_1}^{t} \left| \frac{1}{M_p^i} f_i(X_i) - \frac{1}{M_p^i} f^*(X_i) \right|}_{(I)} + 
    \\ &+\underbrace{\frac{1}{t}\sum_{i=N_1}^{t}\left|\frac{1}{M_p^i} f^*(X_i) - \frac{1}{M_p}f^*(X_i) \right|}_{(II)}.
\end{align*}

We now handle terms (I) and (II) separately. For term (II), we observe that
\begin{align*}
    \sup_{t \geq 2N_1}\frac{1}{t}\sum_{i=N_1}^{t}\left|\frac{1}{M_p^i} f^*(X_i) - \frac{1}{M_p}f^*(X_i) \right|
    &= \sup_{t \geq  2N_1}\frac{1}{t}\sum_{i=N_1}^{t}\left| \frac{    1}{M_p^i} - \frac{1}{M_p} \right| \left|f^*(X_i)\right|
    \\&\leq \lp\sup_{t \geq 2N_1}\left|\frac{1}{M_p^i} - \frac{1}{M_p}\right| \rp\lp\sup_{t \geq 2N_1} \frac{1}{t}\sum_{i=N_1}^{t} \left|f^*(X_i)\right|\rp
    \\&\leq \lp\sup_{t \geq N_1}\left|\frac{1}{M_p^i} - \frac{1}{M_p}\right| \rp\lp\sup_{t \geq 2N_1} \frac{1}{t-N_1}\sum_{i=N_1}^{t} \left|f^*(X_i)\right|\rp.
\end{align*}
Note that this is upper bounded by $ \frac{\epsilon}{4B} B = \frac{\epsilon}{4}$ with probability $1 - \frac{2\delta}{10}$ in view of the union bound. 

For term (I), we have that
\begin{align*}
    \sup_{t \geq 2N_1}\frac{1}{t}\sum_{i=N_1}^{t}  \left| \frac{1}{M_p^i} f_i(X_i) - \frac{1}{M_p^i} f^*(X_i) \right| &\leq \lp\sup_{t \geq N_1}\left|\frac{1}{M_p^i}\right| \rp\lp\sup_{t \geq 2N_1} \frac{1}{t}\sum_{i=N_1}^{t} \left|f_i(X_i) - f^*(X_i)\right|\rp.
\end{align*}

Now we observe that
\begin{align*}
    \sup_{t \geq 2N_1} \frac{1}{t}\sum_{i=N_1}^{t} \left|f_i(X_i) - f^*(X_i)\right| &= \sup_{t \geq 2N_1} \frac{1}{t}\sum_{i=N_1}^{t} \left\langle\frac{1}{i-1} \sum_{k = 1}^{i-1} \xi_p (\cdot, X_k) - \E[\xi_p (\cdot, X)], \xi_p(\cdot, X_i)  \right\rangle_{\mathcal{H}^d}
    \\&\leq \sup_{t \geq 2N_1} \frac{1}{t}\sum_{i=N_1}^{t} \left\|\frac{1}{i-1} \sum_{k = 1}^{i-1} \xi_p (\cdot, X_k) - \E[\xi_p (\cdot, X)] \right\|_{\mathcal{H}^d} \left\| \xi_p(\cdot, X_i)  \right\|_{\mathcal{H}^d}
    \\&\leq \sup_{i \geq 2N_1} \left\{ \left\|\frac{1}{i-1} \sum_{k = 1}^{i-1} \xi_p (\cdot, X_k) - \E[\xi_p (\cdot, X)] \right\|_{\mathcal{H}^d} \right\} \sup_{t \geq 2N_1} \left\{ \frac{1}{t}\sum_{i=N_1}^{t}  \left\| \xi_p(\cdot, X_i)  \right\|_{\mathcal{H}^d}\right\}
    \\&\leq \sup_{i \geq N_1} \left\{ \left\|\frac{1}{i-1} \sum_{k = 1}^{i-1} \xi_p (\cdot, X_k) - \E[\xi_p (\cdot, X)] \right\|_{\mathcal{H}^d} \right\} \\&\quad \times\sup_{t \geq 2N_1} \left\{ \frac{1}{t-N_1}\sum_{i=N_1}^{t}  \left\| \xi_p(\cdot, X_i)  \right\|_{\mathcal{H}^d}\right\}
\end{align*}
Note that this is upper bounded by $\frac{\epsilon}{4B^2} B = \frac{\epsilon}{4B}$ with probability $\frac{2\delta}{10}$ in view of the union bound. Thus term (I) is upper bounded by $\frac{\epsilon}{4B} B = \frac{\epsilon}{4}$  with probability $1 - \frac{3\delta}{10}$ in view of the union bound.

\textit{Concluding the step by considering the union bound over all the terms:} Taking $M:= \max(2N_1, N_2)$, we conclude that
\begin{align*}
    \sup_{t \geq M} \left|\frac{1}{t}\sum_{i=1}^t g_i(X_i) - \frac{1}{t}\sum_{i=1}^t g^*(X_i) \right| \leq \frac{\epsilon}{2} + \frac{\epsilon}{4} + \frac{\epsilon}{4} = \epsilon
\end{align*}
with probability $1 - (\frac{\delta}{2} + \frac{2\delta}{10} + \frac{3\delta}{10}) = 1 - \delta$ in view of the union bound.

\end{proof}

\begin{proposition} \label{proposition:second_order}
If $\E[h_p(X, X)] < \infty$ and $M_p < \infty$, then
\begin{align*}
    \frac{1}{t}\sum_{i=1}^t g_i^2(X_i) \stackrel{a.s.}{\to} \E[\lp g^*(X)\rp^2].
\end{align*} 
\end{proposition}

\begin{proof}

Without loss of generality, we can further assume that $M_p > 0$. We first highlight that, based on $0 < \E[h_p(X, X)] < \infty$,
\begin{itemize}
    \item both $\E[\lba f^*(X) \rba]$ and $\E[(f^*(X))^2]$ are finite, as $\E[(f^*(X))^2] = \E_X[ \E_{X'}[ h^2_p(X', X)]] = \E_{X, X'}[ h^2_p(X', X)] \leq \E_{X, X'}\left[ \left\| \xi_p(\cdot, X)  \right\|_{\mathcal{H}^d}^2  \left\| \xi_p(\cdot, X')  \right\|_{\mathcal{H}^d}^2 \right] = \E_{X}\left[ \left\| \xi_p(\cdot, X)  \right\|_{\mathcal{H}^d}^2 \right]^2 = \E[h_p(X, X)]^2 < \infty$ and $\E^2[f^*(X)] \leq \E[(f^*(X))^2]$,
    \item both $\E\lb\left\| \xi_p(\cdot, X)  \right\|_{\mathcal{H}^d} \rb$ and $\E\lb\left\| \xi_p(\cdot, X)  \right\|^2_{\mathcal{H}^d} \rb$ is finite, as $\E_{X}\left[ \left\| \xi_p(\cdot, X)  \right\|_{\mathcal{H}^d}^2 \right] = \E[h_p(X, X)] < \infty$ and $\E^2\lb\left\| \xi_p(\cdot, X)  \right\|_{\mathcal{H}^d} \rb \leq \E\lb\left\| \xi_p(\cdot, X)  \right\|^2_{\mathcal{H}^d} \rb$.
    \item both $\E[\lba g^*(X)\rba]$ and $\E[\lp g^*(X)\rp^2]$ are finite, as $g^*(x) = f^*(x) / M_p$, $\E[\lba f^*(X) \rba]$ and $\E[\lp f^*(X) \rp^2]$ are also finite, and $M_p > 0$.
\end{itemize}

We want to show that $\frac{1}{t}\sum_{i=1}^t g_i^2(X_i) \stackrel{a.s.}{\to} \E[(g^*(X))^2]$, i.e. $\frac{1}{t}\sum_{i=1}^t g_i^2(X_i) - \E[(g^*(X))^2] \stackrel{a.s.}{\to} 0$. We decompose 
\begin{align*}
    \frac{1}{t}\sum_{i=1}^t g_i^2(X_i) - \E[(g^*(X))^2] &= \frac{1}{t}\sum_{i=1}^t \lp g_i(X_i) - g^*(X_i) + g^*(X_i)\rp^2 - \E[(g^*(X))^2] 
    \\&= \frac{1}{t}\sum_{i=1}^t \lp g_i(X_i) - g^*(X_i)\rp^2 - \frac{2}{t}\sum_{i=1}^t \lp g_i(X_i) - g^*(X_i)\rp g^*(X_i) + 
    \\&+ \frac{1}{t}\sum_{i=1}^t \lp g^*(X_i)\rp^2 - \E[(g^*(X))^2].
\end{align*}
The third term converges to zero almost surely in view of $\E[\lp g^*(X) \rp^2] < \infty$ and the SLLN. For the second term, we apply Cauchy-Schwarz inequality to obtain 
\begin{align*}
    \lba \frac{2}{t}\sum_{i=1}^t \lp g_i(X_i) - g^*(X_i)\rp g^*(X_i) \rba &\leq  \frac{2}{t}\sum_{i=1}^t\lba \lp g_i(X_i) - g^*(X_i)\rp g^*(X_i) \rba
    \\&\leq 2 \lb \frac{1}{t}\sum_{i=1}^t \lp g_i(X_i) - g^*(X_i)\rp^2 \rb^{\frac{1}{2}} \lb \frac{1}{t} \sum_{i=1}^t \lp g^*(X_i) \rp^2 \rb^{\frac{1}{2}}
\end{align*}

Given that $\frac{1}{t} \sum_{i=1}^t \lp g^*(X_i) \rp^2 \cas \E[\lp g^*(X) \rp^2]$, we derive
\begin{align*}
    \lb \frac{1}{t} \sum_{i=1}^t \lp g^*(X_i) \rp^2 \rb^{\frac{1}{2}} \cas E^{\frac{1}{2}}[\lp g^*(X) \rp^2].
\end{align*}
Note that if $\frac{1}{t}\sum_{i=1}^t \lp g_i(X_i) - g^*(X_i)\rp^2 \cas 0$, it also follows that $\lb \frac{1}{t}\sum_{i=1}^t \lp g_i(X_i) - g^*(X_i)\rp^2 \rb^{\frac{1}{2}} \cas 0$, implying that the second term converges to zero almost surely. Thus it suffices to show that the first term converges to zero almost surely, i.e.
$\frac{1}{t}\sum_{i=1}^t \lp g_i(X_i) - g^*(X_i)\rp^2 \cas 0$, to conclude the result. We prove so similarly to Proposition \ref{proposition:first_order}.  

\textit{Introducing the SLLN terms and the probabilistic bounds:} By the SLLN and the finiteness of $M_p$, $\E[\lp f^*(X) \rp^2]$, and $\E\lb\left\| \xi_p(\cdot, X)  \right\|_{\mathcal{H}^d}^2 \rb$, we have that
\begin{itemize}
    \item $M_p^i \stackrel{a.s.}{\to} M_p$, and so $\frac{    1}{M_p^i} \stackrel{a.s.}{\to} \frac{1}{M_p}$ (given that $M_p \neq 0$), 
    \item $\frac{1}{t}\sum_{i=M_1}^{t} \lp f^*(X_i)\rp^2 \stackrel{a.s.}{\to} \E[\lp f^*(X)\rp^2]$,
    \item $\frac{1}{t}\sum_{i=1}^{t}  \left\| \xi_p(\cdot, X_i)  \right\|_{\mathcal{H}^d}^2 \stackrel{a.s.}{\to} \E\lb\left\| \xi_p(\cdot, X)  \right\|_{\mathcal{H}^d}^2 \rb$.
\end{itemize}

From the Banach space-valued SLLN (Theorem \ref{theorem:slln_banach}) and the finiteness of $\E\lb\left\| \xi_p(\cdot, X)  \right\|_{\mathcal{H}^d} \rb$, it also follows that
\begin{itemize}
    \item $\left\| \frac{1}{i-1} \sum_{k = 1}^{i-1} \xi_p (\cdot, X_k) - \E[\xi_p (\cdot, X)] \right\|_{\mathcal{H}^d} \cas 0$.
\end{itemize}

Hence there exist $B > 0$ and $N_1 \in \Nb$ such that
\begin{align*}
    &\Pb \lp \sup_{t \geq N_1} \lp \frac{1}{M_p^i} \rp^2 > B \rp \leq \frac{\delta}{10},
    \\
    &\Pb \lp \sup_{t \geq N_1} \frac{1}{t}\sum_{i=1}^{t} \lp f^*(X_i)\rp^2 > B \rp \leq \frac{\delta}{10},
    \\
    &\Pb \lp \sup_{t \geq N_1} \frac{1}{t}\sum_{i=1}^{t}  \left\| \xi_p(\cdot, X_i)  \right\|_{\mathcal{H}^d}^2 > B \rp \leq \frac{\delta}{10},
    \\
    &\Pb \lp \sup_{t \geq N_1} \lp \frac{1}{M_p^i} - \frac{1}{M_p} \rp^2 > \frac{\epsilon}{8B} \rp \leq \frac{\delta}{10}.
    \\
    &\Pb \lp \sup_{i \geq N_1} \left\| \frac{1}{i-1} \sum_{k = 1}^{i-1} \xi_p (\cdot, X_k) - \E[\xi_p (\cdot, X)] \right\|_{\mathcal{H}^d}^2 > \frac{\epsilon}{8B^2} \rp \leq \frac{\delta}{10}.
\end{align*}

Note that $\Pb \lp \sup_{t \geq N_1} \frac{1}{t}\sum_{i=1}^{t}  \left\| \xi_p(\cdot, X_i)  \right\|_{\mathcal{H}^d}^2 > B \rp \leq \frac{\delta}{10}$ and $\Pb \lp \sup_{t \geq N_1} \frac{1}{t}\sum_{i=1}^{t} \lp f^*(X_i)\rp^2 > B \rp \leq \frac{\delta}{10}$ imply that $\Pb \lp \sup_{t \geq 2N_1} \frac{1}{t-N_1}\sum_{i=N_1}^{t}  \left\| \xi_p(\cdot, X_i)  \right\|_{\mathcal{H}^d}^2 > B \rp \leq \frac{\delta}{5}$ and $\Pb \lp \sup_{t \geq 2N_1} \frac{1}{t-N_1}\sum_{i=N_1}^{t} \lp f^*(X_i)\rp^2 > B \rp \leq \frac{\delta}{10}$, given that the data are iid.

\textit{Combining the probabilistic bounds:} We start by noting that
\begin{align*}
     \frac{1}{t}\sum_{i=1}^t \lp g_i(X_i) - g^*(X_i) \rp^2 &= \frac{1}{t}\sum_{i=1}^{N_1-1} \lp g_i(X_i) - g^*(X_i) \rp^2 + \frac{1}{t}\sum_{i=N_1}^{t} \lp g_i(X_i) - g^*(X_i) \rp^2.
\end{align*}

Consider the sequence of random variables $Y_t = \frac{1}{t}\left(\sum_{i=1}^{N_1-1} \lp g_i(X_i) - g^*(X_i) \rp^2\right)$. Clearly, $Y_t \stackrel{a.s.}{\to} 0$. Thus there exists $N_2$ such that $\sup_{t \geq N_2} Y_t \leq \frac{\epsilon}{2}$ with probability $1 - \delta / 2$.

Moreover,
\begin{align*}
    \frac{1}{t}\sum_{i=N_1}^t \lp g_i(X_i) - g^*(X_i)\rp^2 &= \frac{1}{t}\sum_{i=N_1}^{t} \left( \frac{1}{M_p^i} f_i(X_i) - \frac{1}{M_p^i} f^*(X_i) + \frac{1}{M_p^i} f^*(X_i) - \frac{1}{M_p}f^*(X_i) \right)^2
    \\ &\stackrel{(i)}{\leq} \underbrace{\frac{2}{t}\sum_{i=N_1}^{t} \left( \frac{1}{M_p^i} f_i(X_i) - \frac{1}{M_p^i} f^*(X_i) \right)^2}_{(I)} + 
    \\ &+\underbrace{\frac{2}{t}\sum_{i=N_1}^{t}\left(\frac{1}{M_p^i} f^*(X_i) - \frac{1}{M_p}f^*(X_i) \right)^2}_{(II)},
\end{align*}
where (i) is obtained in view of $(a + b)^2 \leq 2a^2 + 2b^2$. It now remains to prove that terms (I) and (II) converge almost surely to zero. For term (II), we observe that
\begin{align*}
    2\sup_{t \geq 2N_1}\frac{1}{t}\sum_{i=N_1}^{t}\lp\frac{1}{M_p^i} f^*(X_i) - \frac{1}{M_p}f^*(X_i) \rp^2
    &= 2\sup_{t \geq 2 N_1}\frac{1}{t}\sum_{i=N_1}^{t}\lp \frac{1}{M_p^i} - \frac{1}{M_p} \rp^2 \lp f^*(X_i)\rp^2
    \\&\leq 2\lp\sup_{t \geq 2N_1}\lp\frac{1}{M_p^i} - \frac{1}{M_p}\rp^2 \rp\lp\sup_{t \geq 2N_1} \frac{1}{t}\sum_{i=N_1}^{t} \lp f^*(X_i)\rp^2\rp
    \\&\leq 2\lp\sup_{t \geq N_1}\lp\frac{1}{M_p^i} - \frac{1}{M_p}\rp^2 \rp\lp\sup_{t \geq 2N_1} \frac{1}{t-N_1}\sum_{i=N_1}^{t} \lp f^*(X_i)\rp^2\rp.
\end{align*}
Note that this is upper bounded by $2 \frac{\epsilon}{8B} B = \frac{\epsilon}{4}$ with probability $1 - \frac{2\delta}{10}$ in view of the union bound. 

For term (I), we have that
\begin{align*}
    2\sup_{t \geq 2N_1}\frac{1}{t}\sum_{i=N_1}^{t}  \lp \frac{1}{M_p^i} f_i(X_i) - \frac{1}{M_p^i} f^*(X_i) \rp^2 &\leq 2\lp\sup_{t \geq N_1}\lp\frac{1}{M_p^i}\rp^2\rp\lp\sup_{t \geq 2N_1} \frac{1}{t}\sum_{i=N_1}^{t} \lp f_i(X_i) - f^*(X_i)\rp^2\rp.
\end{align*}

Now we observe that
\begin{align*}
    \sup_{t \geq 2N_1} \frac{1}{t}\sum_{i=N_1}^{t} \lp f_i(X_i) - f^*(X_i)\rp^2 &= \sup_{t \geq 2N_1} \frac{1}{t}\sum_{i=N_1}^{t} \left\langle \E[\xi_p (\cdot, X)] - \frac{1}{i-1} \sum_{k = 1}^{i-1} \xi_p (\cdot, X_k), \xi_p(\cdot, X_i)  \right\rangle_{\mathcal{H}^d}^2
    \\&\leq \sup_{t \geq 2N_1} \frac{1}{t}\sum_{i=N_1}^{t} \left\| \E[\xi_p (\cdot, X)] - \frac{1}{i-1} \sum_{k = 1}^{i-1} \xi_p (\cdot, X_k) \right\|_{\mathcal{H}^d}^2\left\|\xi_p(\cdot, X_i)  \right\|_{\mathcal{H}^d}^2
    \\&\leq \sup_{t \geq 2N_1}\left\| \E[\xi_p (\cdot, X)] - \frac{1}{i-1} \sum_{k = 1}^{i-1} \xi_p (\cdot, X_k) \right\|_{\mathcal{H}^d}^2 \sup_{t \geq 2N_1}\frac{1}{t}\sum_{i=N_1}^{t} \left\|\xi_p(\cdot, X_i)  \right\|_{\mathcal{H}^d}^2
    \\&\leq \sup_{t \geq N_1}\left\| \E[\xi_p (\cdot, X)] - \frac{1}{i-1} \sum_{k = 1}^{i-1} \xi_p (\cdot, X_k) \right\|_{\mathcal{H}^d}^2 \sup_{t \geq 2N_1}\frac{1}{t-N_1}\sum_{i=N_1}^{t} \left\|\xi_p(\cdot, X_i)  \right\|_{\mathcal{H}^d}^2
\end{align*}

Note that this is upper bounded by $\frac{\epsilon}{8B^2} B = \frac{\epsilon}{8}B$ with probability $1 - \frac{2\delta}{10}$ in view of the union bound, and hence term (I) is upper bounded by $2\frac{\epsilon}{8B} B = \frac{\epsilon}{4}$ with probability $1 - \frac{3\delta}{10}$ in view of the union bound.

\textit{Concluding the step by considering the union bound over all the terms:}

Taking $N:= \max(2N_1, N_2)$, we conclude that
\begin{align*}
    \sup_{t \geq N} \left|\frac{1}{t}\sum_{i=1}^t g_i(X_i) - \frac{1}{t}\sum_{i=1}^t g^*(X_i) \right| \leq \frac{\epsilon}{2} + \frac{\epsilon}{4} + \frac{\epsilon}{4} = \epsilon
\end{align*}
with probability $1 - (\frac{\delta}{2} + \frac{2\delta}{10} + \frac{3\delta}{10}) = 1 - \delta$ in view of the union bound.

\end{proof}

\subsection{Proofs of Main Theorems}

\begin{proof}[Proof of Theorem \ref{theorem:ksd_simple}]

We first note that
\begin{align*}
\mathbb{E}_{H_0}\left[ g_t(X_t) \big|\F_{t-1}\right] &= \mathbb{E}_{H_0}\lb \frac{1}{\frac{1}{t-1}\sum_{i=1}^{t-1} M_p(X_i)}\left(\frac{1}{t-1} \sum_{i = 1}^{t-1} h_p(X_i, X_t)\right) \bigg|\F_{t-1}\rb
\\ &=  \frac{1}{\frac{1}{t-1}\sum_{i=1}^{t-1} M_p(X_i)} \mathbb{E}_{H_0}\left[ \frac{1}{t-1} \sum_{i = 1}^{t-1} h_p(X_i, X_t)\bigg|\F_{t-1}\right]
\\ &\stackrel{(i)}{=}  \frac{1}{\frac{1}{t-1}\sum_{i=1}^{t-1} M_p(X_i)} \times 0
\\&= 0,
\end{align*}
where (i) is obtained in view of \eqref{eq:ksd_zero_mean}.

Hence
\begin{align*}
\mathbb{E}_{H_0}\left[\wealth_t\big|\F_{t-1}\right] &= \mathbb{E}_{H_0}\left[\wealth_{t-1} \times \lp 1 + \lambda_t g_t(X_t)  \rp\big|\F_{t-1}\right]
\\ 
&= \wealth_{t-1} \times \lp 1 + \lambda_t\mathbb{E}_{H_0}\left[ g_t(X_t) \big|\F_{t-1}\right] \rp
\\ &= \wealth_{t-1} \times \lp 1 + 0  \rp 
\\&= \wealth_{t-1},
\end{align*}
which implies that $\wealth_{t-1}$ is a martingale. Furthermore, 
\begin{align*}
    g_t(x) &=  \frac{1}{\frac{1}{t-1}\sum_{i=1}^{t-1} M_p(X_i)}\left(\frac{1}{t-1} \sum_{i = 1}^{t-1} h_p(X_i, x)\right)
     \\&\geq \frac{1}{\frac{1}{t-1}\sum_{i=1}^{t-1} M_p(X_i)}\left( \frac{1}{t-1}\sum_{i=1}^{t-1} -M_p (X_i)\right)
    \\&= -1.
\end{align*}

Given that $\lambda_t \in [0, 1]$, $\lambda_t g_t$ is also lower bounded by $-1$, and so $\wealth_t$ is non-negative. Thus, $\wealth_t$ is a test martingale. In view of $\mathbb{E}_{H_0}[\wealth_0] = 1$ and Ville's inequality, we conclude that $\tau$ is a level-$\alpha$ sequential test. 

\end{proof}

\begin{proof}[Proof of Theorem \ref{theorem:ksd_simple_power_lbow}]

We want to prove that $\liminf_{t \to \infty}\frac{\log\wealth_t}{t}  \geq \frac{\left(\E_{H_1}[g^*(X)]\right)^2 / 2}{\E_{H_1}[g^*(X)] + \E_{H_1}[\left(g^*(X)\right)^2]} =: L$.

Following the LBOW strategy, we take
\begin{align*}
    \lambda_t = \max\lp0, \frac{\frac{1}{t-1}\sum_{i = 1}^{t-1} g_i(X_i)}{\frac{1}{t-1}\sum_{i = 1}^{t-1} g_i(X_i) + \frac{1}{t-1}\sum_{i = 1}^{t-1} g_i^2(X_i)}\rp.
\end{align*}

Given that $\E[h_p(X, X)] < \infty$, it also holds that $\E[\sqrt{h_p(X, X)}] < \infty$. Consequently, Proposition \ref{proposition:first_order} and Proposition \ref{proposition:second_order} yield $\frac{1}{t-1}\sum_{i = 1}^{t-1} g_i(X_i) \cas \E[g^*(X)]$, as well as $\frac{1}{t-1}\sum_{i = 1}^{t-1} g_i^2(X_i) \cas \E[\lp g^*(X)\rp^2]$. Given that $\E[g^*(X)] > 0$, it follows that
\begin{align*}
    \lambda_t \cas \frac{\E[g^*(X)]}{\E[g^*(X)] + \E[\lp g^*(X)\rp^2]} =: \lambda^* \in (0, 1).
\end{align*}

For $y \geq -1$ and $\lambda \in [0, 1)$, it holds that 
\begin{align*}
    \log(1 + \lambda y) \geq \lambda y + y^2 \left( \log(1 - \lambda) + \lambda \right).
\end{align*}

Further, for $\lambda \in [0, 1)$, it holds that 
\begin{align*}
    \log(1 - \lambda) + \lambda \geq -\frac{\lambda^2}{2(1 - \lambda)}.
\end{align*}

Thus, for $y \geq -1$ and $\lambda \in [0, 1)$,
\begin{align*}
    \log(1 + \lambda y) \geq \lambda y - y^2\frac{\lambda^2}{2(1 - \lambda)}.
\end{align*}

Consequently, we derive that
\begin{align*}
    \frac{\log\wealth_t}{t} &= \frac{1}{t} \sum_{i = 1}^t \log \lp 1 + \lambda_i g_i(X_i) \rp
    \\ & \geq \frac{1}{t} \sum_{i = 1}^t \lambda_i g_i(X_i) - g_i^2(X_i)\frac{\lambda_i^2}{2(1 - \lambda_i)}.
\end{align*}

It thus suffices to prove that $\frac{1}{t} \sum_{i = 1}^t \lambda_i g_i(X_i) - g_i^2(X_i)\frac{\lambda_i^2}{2(1 - \lambda_i)} \cas L$ to conclude the result. To see this, we denote $\kappa(\lambda) = \frac{\lambda^2}{2(1 - \lambda)}$ and highlight that 
\begin{align*}
    \frac{1}{t} \sum_{i = 1}^t \lambda_i g_i(X_i) - g_i^2(X_i)\frac{\lambda_i^2}{2(1 - \lambda_i)} &= \frac{1}{t} \sum_{i = 1}^t \lambda_i g_i(X_i) - g_i^2(X_i)\kappa(\lambda_i)
    \\&= \frac{1}{t} \sum_{i = 1}^t (\lambda_i - \lambda^* + \lambda^*) g_i(X_i) - \lp\kappa(\lambda_i) - \kappa(\lambda^*) + \kappa(\lambda^*)\rp g_i^2(X_i)
    \\&= \underbrace{\frac{1}{t} \sum_{i = 1}^t  \lambda^* g_i(X_i) -  \kappa(\lambda^*) g_i^2(X_i)}_{(I)} +
   \\ &\quad + \underbrace{ 
    \frac{1}{t} \sum_{i = 1}^t (\lambda_i - \lambda^*) g_i(X_i)}_{(II)} - \underbrace{\frac{1}{t} \sum_{i = 1}^t\lp\kappa(\lambda_i) - \kappa(\lambda^*)\rp g_i^2(X_i)}_{(III)}.
\end{align*}
Now we note that term (I) converges almost surely to $\lambda^* \E[g^*(X)] - \kappa(\lambda^*) \E[\lp g^*(X)\rp^2] = L$, given that $\frac{1}{t}\sum_{i = 1}^{t} g_i(X_i) \cas \E[g^*(X)]$ and $\frac{1}{t}\sum_{i = 1}^{t} g_i^2(X_i) \cas \E[\lp g^*(X)\rp^2]$. Thus, it suffices to prove that (II) and (III) converge almost surely to zero to derive that
\begin{align*}
    \frac{1}{t} \sum_{i = 1}^t \lambda_i g_i(X_i) - g_i^2(X_i)\frac{\lambda_i^2}{2(1 - \lambda_i)} \cas L,
\end{align*}
hence concluding that 
\begin{align*}
    \liminf_{t\to\infty}\frac{\log\wealth_t}{t} &\geq\liminf_{t\to\infty}  \frac{1}{t} \sum_{i = 1}^t \lambda_i g_i(X_i) - g_i^2(X_i)\frac{\lambda_i^2}{2(1 - \lambda_i)} = L.
\end{align*}

Let us now prove that (II) and (III) converge almost surely to 0. Let $\epsilon > 0$ and $\delta > 0$ be arbitrary. We ought to show that there exists $N \in \mathbb{N}$ such that
\begin{align*}
    \Pb \lp \sup_{t \geq N} \left|\frac{1}{t} \sum_{i = 1}^t (\lambda_i - \lambda^*) g_i(X_i) \right| > \epsilon \rp \leq \delta, \quad  \Pb \lp \sup_{t \geq N} \left|\frac{1}{t} \sum_{i = 1}^t\lp\kappa(\lambda_i) - \kappa(\lambda^*)\rp g_i^2(X_i) \right| > \epsilon \rp \leq \delta.
\end{align*}

Based on $\frac{1}{t}\sum_{i=1}^t \left| g_i(X_i) - g^*(X_i) \right| \stackrel{a.s.}{\to} 0$, we derive that that
\begin{align*}
    \frac{1}{t} \sum_{i = 1}^t \left|  g_i(X_i) \right| &= \frac{1}{t} \sum_{i = 1}^t \left|  g_i(X_i) - g^*(X_i) + g^*(X_i) \right|
    \\&\leq \underbrace{\frac{1}{t} \sum_{i = 1}^t \lba  g_i(X_i) - g^*(X_i) \rba}_{\cas 0}+ \underbrace{\frac{1}{t} \sum_{i = 1}^t\lba g^*(X_i) \rba}_{\cas \E[\lba g^*(X)\rba]},
\end{align*}
and so $\limsup_{t\to \infty}\frac{1}{t} \sum_{i = 1}^t \left|  g_i(X_i) \right| \leq \E[\lba g^*(X)\rba]$. Furthermore, we know that $\frac{1}{t}\sum_{i = 1}^{t} g_i^2(X_i) \cas \E[\lp g^*(X)\rp^2]$. Given that $\lambda \cas \lambda^*$ and $\kappa$ is continuous, it follows that $\kappa(\lambda) \cas \kappa(\lambda^*)$. Thus, by the SLLN, there exist $B>0$ and $N_1 \in \Nb$ such that
\begin{align*}
&\Pb \lp \sup_{t \geq N_1}\frac{1}{t} \sum_{i = 1}^t \left|  g_i(X_i) \right| > B \rp \leq \frac{\delta}{3}, \quad \Pb \lp \sup_{t \geq N_1}\frac{1}{t} \sum_{i = 1}^t   g_i^2(X_i) > B \rp \leq \frac{\delta}{3},
    \\&\Pb \lp \sup_{t \geq N_1} \left|\lambda_i - \lambda^*\right| > \frac{\epsilon}{2B} \rp \leq \frac{\delta}{3}, \quad \Pb \lp \sup_{t \geq N_1} \left|\kappa(\lambda_i) - \kappa(\lambda^*)\right| > \frac{\epsilon}{2B} \rp \leq \frac{\delta}{3}.
\end{align*}

Note that $\Pb \lp \sup_{t \geq N_1}\frac{1}{t} \sum_{i = 1}^t \left|  g_i(X_i) \right| > B \rp \leq \frac{\delta}{3}$ and $\Pb \lp \sup_{t \geq N_1}\frac{1}{t} \sum_{i = 1}^t   g_i^2(X_i) > B \rp \leq \frac{\delta}{3}$ imply that $\Pb \lp \sup_{t \geq 2N_1}\frac{1}{t-N_1} \sum_{i = N_1}^t \left|  g_i(X_i) \right| > B \rp \leq \frac{\delta}{3}$ and $\Pb \lp \sup_{t \geq 2N_1}\frac{1}{t-N_1} \sum_{i = N_1}^t   g_i^2(X_i) > B \rp \leq \frac{\delta}{3}$, given that the data are iid.

Thus,
\begin{align*}
    \sup_{t \geq 2N_1} \left|\frac{1}{t} \sum_{i = N_1}^t (\lambda_i - \lambda^*) g_i(X_i) \right| &\leq  \sup_{t \geq 2N_1} \frac{1}{t} \sum_{i = N_1}^t \left| (\lambda_i - \lambda^*) g_i(X_i) \right|
    \\&\leq  \sup_{t \geq N_1} \left| \lambda_i - \lambda^*\right| \sup_{t \geq 2N_1} \frac{1}{t} \sum_{i = N_1}^t \left|  g_i(X_i) \right|
    \\&\leq  \sup_{t \geq N_1} \left| \lambda_i - \lambda^*\right| \sup_{t \geq 2N_1} \frac{1}{t-N_1} \sum_{i = N_1}^t \left|  g_i(X_i) \right|
    \\&\leq \frac{\epsilon}{2B}B
    \\&= \frac{\epsilon}{2}
\end{align*}
with probability $1 - \frac{2}{3} \delta$, as well as 
\begin{align*}
   \sup_{t \geq 2N_1} \left|\frac{1}{t} \sum_{i = N_1}^t\lp\kappa(\lambda_i) - \kappa(\lambda^*)\rp g_i^2(X_i) \right| &\leq \sup_{t \geq 2N_1} \frac{1}{t} \sum_{i = N_1}^t\left|\lp\kappa(\lambda_i) - \kappa(\lambda^*)\rp g_i^2(X_i) \right|
   \\&\leq \sup_{t \geq 2N_1} \left|\kappa(\lambda_i) - \kappa(\lambda^*) \right|
   \sup_{t \geq 2N_1} \frac{1}{t} \sum_{i = N_1}^t g_i^2(X_i)
   \\&\leq \sup_{t \geq N_1} \left|\kappa(\lambda_i) - \kappa(\lambda^*) \right|
   \sup_{t \geq 2N_1} \frac{1}{t-N_1} \sum_{i = N_1}^t g_i^2(X_i)
    \\&\leq \frac{\epsilon}{2B}B
    \\&= \frac{\epsilon}{2}
\end{align*}
with probability $1 - \frac{2}{3} \delta$.

Consider now the sequence of random variables 
\begin{align*}
    Y_t = \frac{1}{t} \sum_{i = 1}^{N_1-1} (\lambda_i - \lambda^*) g_i(X_i), \quad \tilde Y_t = \frac{1}{t} \sum_{i = 1}^{N_1 - 1}\lp\kappa(\lambda_i) - \kappa(\lambda^*)\rp g_i^2(X_i).
\end{align*}

Clearly, $Y_t \stackrel{a.s.}{\to} 0$ and $\tilde Y_t \stackrel{a.s.}{\to} 0$. Thus there exists $N_2$ such that $Y_t \leq \frac{\epsilon}{2}$ and $\tilde Y_t \leq \frac{\epsilon}{3}$, both with probability $1 - \frac{\delta}{3}$. Taking $N = \max(2N_1, N_2)$ and in view of the union bound, we derive that 
\begin{align*}
    \sup_{t \geq N} \left|\frac{1}{t} \sum_{i = 1}^t (\lambda_i - \lambda^*) g_i(X_i) \right|  &\leq \sup_{t \geq N} \left|\frac{1}{t} \sum_{i = 1}^{N_1 - 1} (\lambda_i - \lambda^*) g_i(X_i) + \frac{1}{t} \sum_{i = N_1}^t (\lambda_i - \lambda^*) g_i(X_i) \right|
    \\&\leq \sup_{t \geq N} \left|\frac{1}{t} \sum_{i = 1}^{N_1 - 1} (\lambda_i - \lambda^*) g_i(X_i)\right| + \sup_{t \geq N} \left| \frac{1}{t} \sum_{i = N_1}^t (\lambda_i - \lambda^*) g_i(X_i) \right|
    \\&\leq \frac{\epsilon}{2} + \frac{\epsilon}{2}
    \\&= \epsilon
\end{align*}
with probability $1 - (\frac{2}{3}\delta+ \frac{1}{3}\delta) = 1-\delta$, as well as
\begin{align*}
    \sup_{t \geq N} \left|\frac{1}{t} \sum_{i = 1}^t\lp\kappa(\lambda_i) - \kappa(\lambda^*)\rp g_i^2(X_i) \right|  &\leq \sup_{t \geq N} \left|\frac{1}{t} \sum_{i = 1}^{N_1 - 1} \lp\kappa(\lambda_i) - \kappa(\lambda^*)\rp g_i^2(X_i) + \frac{1}{t} \sum_{i = N_1}^t\lp\kappa(\lambda_i) - \kappa(\lambda^*)\rp g_i^2(X_i) \right|
    \\&\leq \sup_{t \geq N} \left|\frac{1}{t} \sum_{i = 1}^{N_1 - 1} \lp\kappa(\lambda_i) - \kappa(\lambda^*)\rp g_i^2(X_i)\right| + \sup_{t \geq N} \left| \frac{1}{t} \sum_{i = N_1}^t \lp\kappa(\lambda_i) - \kappa(\lambda^*)\rp g_i^2(X_i)\right|
    \\&\leq \frac{\epsilon}{2} + \frac{\epsilon}{2}
    \\&= \epsilon
\end{align*}
with probability $1 - (\frac{2}{3}\delta+ \frac{1}{3}\delta) = 1-\delta$.
\end{proof}

\begin{proof}[Proof of Theorem \ref{theorem:ksd_composite}]

Let $\theta_0 \in \Theta$ be such that $P_{\theta_0} = P$. By definition of $\wealth_t^C$, we have that $\wealth_t^C \leq \wealth_t^{\theta_0}$, with $\wealth_t^{\theta_0}$ being a test martingale. The validity of the test is concluded by noting that 
\begin{align*}
    \Pb \lp \wealth_t^C \geq \frac{1}{\alpha}\rp \leq \Pb \lp \wealth_t^{\theta_0} \geq \frac{1}{\alpha}\rp \leq \alpha,
\end{align*}
where the second inequality is obtained in view of Ville's inequality.

\end{proof}

\end{document}